\documentclass[letterpaper]{article} % DO NOT CHANGE THIS
\usepackage{aaai24}  % DO NOT CHANGE THIS
\usepackage{times}  % DO NOT CHANGE THIS
\usepackage{helvet}  % DO NOT CHANGE THIS
\usepackage{courier}  % DO NOT CHANGE THIS
\usepackage[hyphens]{url}  % DO NOT CHANGE THIS
\usepackage{graphicx} % DO NOT CHANGE THIS
\usepackage{natbib}  % DO NOT CHANGE THIS AND DO NOT ADD ANY OPTIONS TO IT
\usepackage{caption} % DO NOT CHANGE THIS AND DO NOT ADD ANY OPTIONS TO IT
\usepackage[basic]{complexity}

\usepackage{soul}
\usepackage{url}
\usepackage{amsmath}
\usepackage{booktabs}
\usepackage{subcaption}

\usepackage{amsthm}
\usepackage{amsfonts}
\usepackage{amssymb}
\usepackage{todonotes}
\usepackage{enumerate}

\def\ba{\begin{array}}
\def\ea{\end{array}}
\def\be{\begin{enumerate}}
\def\ee{\end{enumerate}}
\def\bi{\begin{itemize}}
\def\ei{\end{itemize}}
\def\beq{\begin{equation}}
\def\eeq#1{\label{#1}\end{equation}}
\def\beeq{\begin{equation*}}
\def\eeeq{\end{equation*}}
\def\nmodels{\,{\nvDash}\,}
\def\mi#1{\mathit{#1\/}}

\def\AS{\mathit{AS}}
\def\HT{\mathit{SE}}
\def\SE{\mathit{SE}}

\def\SP{\textup{\textbf{(SP)}}}

\def\rSP{\textup{\textbf{(rSP)}}}
\def\wC{\textup{\textbf{(wC)}}}
\def\CP{\textup{\textbf{(CP)}}}

\newenvironment{myi}
{\begin{list}{}{%
\setlength{\topsep}{4pt} %0pt
\setlength{\leftmargin}{5pt} %5pt or variant2: %0pt or v3:10
\setlength{\itemindent}{\labelwidth}}}
{\end{list}}

\newcommand{\la}{\leftarrow}
\newcommand{\htm}[1]{\langle #1 \rangle}
\newcommand{\Lits}{\mathcal{U}}
\newcommand{\nop}[1]{}

\newcommand{\mycomment}[1]{{\textbf{\color{blue}{*** #1 ***}}}}

\newtheorem{thm}{Theorem}

\newtheorem{lemma}[thm]{Lemma}
\newtheorem{prop}[thm]{Proposition}
\newtheorem{cor}[thm]{Corollary}

\newtheorem{defn}{Definition}
\newtheorem{exmp}{Example}

\title{A Unified View on Forgetting and Strong Equivalence Notions\\ in Answer Set Programming}
\author{
    Zeynep G. Saribatur, Stefan Woltran
}
\affiliations{
    Institute of Logic and Computation, TU Wien, Austria\\
    \{zeynep.saribatur,stefan.woltran\}@tuwien.ac.at
}

\begin{document}

\maketitle

\begin{abstract}
  Answer Set Programming (ASP) is a prominent rule-based language for knowledge representation and reasoning with roots in logic programming and non-monotonic reasoning. The aim to capture the essence of removing (ir)relevant details in ASP programs led to the investigation of different notions, from strong persistence \SP\ forgetting, to faithful abstractions, and, recently, strong simplifications, where the latter two can be seen as relaxed and strengthened notions of forgetting, respectively. Although it was observed that these notions are related, especially given that they have characterizations through the semantics for strong equivalence, it remained unclear whether they can be brought together. In this work, we bridge this gap by introducing a novel relativized equivalence notion, which is a relaxation of the recent simplification notion, that is able to capture all related notions from the literature. We provide necessary and sufficient conditions for relativized simplifiability, which shows that the challenging part is for when the context programs do not contain all the atoms to remove. We then introduce an operator that combines projection and a relaxation of \SP-forgetting to obtain the relativized simplifications. We furthermore present complexity results that complete the overall picture.

\end{abstract}

\section{Introduction}
\nop{
\todo[inline]{
-forgetting/removal is an interesting problem

-nonmonotonic nature of ASP makes things more interesting

-many works are conducted for forgetting in ASP

-SP forgetting for ASP is important, which is inspired from SE

-relaxations of forgetting has also been investigated through the view of abstraction (related to (wC))

-faithful abstractions

-recently a stronger notion of abstraction has been considered, strong simplifications, again inspired from SE

-it is known that strong simplifications imply SP forgetting, and some work has been conducted to relate forgetting with omission abstraction.

-however there is the lack of a general notion that is able to capture all of the above mentioned concepts

}
}

%\cite{zgskr20}\cite{goncalves_knorr_leite_2023}\cite{DBLP:journals/ai/GoncalvesKLW20}

Forgetting or discarding information that are not deemed necessary is crucial in human reasoning, as it allows to focus on the important details and to abstract over the rest. Such active or \emph{intentional} forgetting is argued to enhance decision-making through flexibility under changing conditions and the ability to generalize \cite{richards2017persistence}. 
Over the years, the desire to abstract over details led to different theories (e.g., \cite{giunchiglia1992theory}) and applications of abstraction in various areas of AI, among many are planning \cite{knoblock1994automatically}, constraint satisfaction \cite{bistarelli2002abstracting}, and model checking \cite{clarkeabstraction94}.
Getting rid of (ir)relevant details through forgetting continues to motivate works in different subfields of AI \cite{beierle2019intentional}, such as knowledge representation and reasoning (KR) \cite{eite-kern-forget-ki-18} and symbolic machine learning \cite{siebers2019please}. Recent examples of forgetting within KR appear in action theories \cite{DBLP:conf/ecai/Luo0LL20}, explanations for planning \cite{vasileiou2022generating} and argumentation \cite{DBLP:conf/kr/BertholdRU23,DBLP:conf/ijcai/BaumannB22}.

The theoretical underpinnings of forgetting has been investigated for classical logic and logic programming for over decades. Answer Set Programming (ASP), is a well established logic programming language, characterized by non-monotonic declarative semantics. Its non-monotonic nature resulted in various forgetting operators satisfying different desirable properties (see recent survey \cite{goncalves_knorr_leite_2023}). %Forgetting techniques is also receiving recent attention in argumentation \cite{DBLP:conf/kr/BertholdRU23,DBLP:conf/ijcai/BaumannB22}.
The property \emph{strong persistence} \SP\ \cite{knorr2014preserving} is considered to best capture the essence of forgetting in the context of ASP. The aim is to preserve all existing relations between the remaining atoms, by requiring that there be a correspondence between the answer sets of a program before and after
forgetting a set of atoms, which is preserved in the presence of additional rules. This correspondence is formally defined as
\begin{equation}\label{eq:a2}
\AS(P \cup R)_{|\overline{A}} = \AS(f(P,A) \cup R)
\end{equation}
for all programs $R$ over the universe $\Lits$ without containing atoms from $A$, where $f(P,A)$ is the resulting program of applying an operator $f$ on $P$ to forget about the set $A$ of atoms, $\AS(\cdot)$ denotes the collection of answer sets of a program, and $\AS(\cdot)_{|\overline{A}} $ is their projection onto the remaining atoms. % $\Lits \setminus A$.

When nothing is forgotten, \SP\ matches the notion of \emph{strong equivalence (SE)} \cite{Lifschitz:2001:SEL:383779.383783}  among programs, denoted as
%\begin{equation}\label{eq:a1}
$\AS(P \cup R)=\AS(Q \cup R)$ for all programs $R$. % over the universe $\Lits$.
%\end{equation}
%Though the aim with forgetting is to reach a program with a reduced signature.
%
%"The non-monotonic rule-based nature of ASP called for the development of speciﬁc methods and techniques – just as it happened with other belief change operations such as revision and update, cf. [35–41] – resulting in a signiﬁcant number of different forgetting operators [15,16,42–46,18], obeying different sets of properties deemed desirable, some adapted from the literature on classical forgetting [47,43,46], others speciﬁcally introduced for the case of ASP [16,42-45,18], and often deﬁned for different classes of answer set programs."
\citeauthor{DBLP:journals/ai/GoncalvesKLW20} \shortcite{DBLP:journals/ai/GoncalvesKLW20} showed that \SP-forgetting can only be done when the SE-models of the program adheres to certain conditions, which is motivated by \emph{relativized} strong equivalence \cite{DBLP:conf/jelia/Woltran04,Eiter05}, a relaxation of strong equivalence where the context programs can exclude some atoms. %show how and when it is possible to do \SP-forgetting.

The motivation to obtain ASP programs with a reduced signature also led to notion of abstraction by omission \cite{zgskr18} % and domain abstraction \cite{DBLP:journals/ai/SaribaturES21},
by means of \emph{over-approximation}, i.e., any answer set in program $P$ can be mapped to some answer set in the abstracted program $Q$, which %. In terms of omission \cite{zgskr18}, this 
is denoted by $\AS(P)_{|\overline{A}} \subseteq \AS(Q)$, and also has been referred as \emph{weakened Consequence} \wC\ within forgetting \cite{gonccalves2016ultimate}. 
%
%Aiming for a relaxed property makes it possible to define a simple syntactic operator to obtain an abstraction.
\citeauthor{zgskr18} 
\shortcite{zgskr18} 
introduce a syntactic operator that obtains abstracted programs, and an automated abstraction and refinement methodology, that starts with a coarse abstraction and refines it upon encountering \emph{spurious} answer sets (which do not have correspondence in $P$) until a fine-grained abstraction is achieved.

A desired abstraction property was considered to be \emph{faithfulness} where $Q$ does not contain a spurious answer set, i.e., %$Q$ is a
%\emph{faithful} abstraction if 
\begin{equation}\label{eq:a3}
\AS(P)_{|\overline{A}} = \AS(Q),
\end{equation}
matching an instance of \emph{Consequence Persistence} \CP-forgetting \cite{wang2013forgetting}. The notion however does not truly preserve the semantics w.r.t. projection. 
The recent equivalence notion, called \emph{strong simplification} 
\cite{zgsswkr23}, 
defined as\footnote{$R_{|\overline{A}}$ projects the positive body of the rules in $R$ onto $\overline{A}$ and removes the rules with a negative body or head containing an atom from $A$.}
\begin{equation}\label{eq:a4}
	\AS(P \cup R)|_{\overline{A}}=\AS(Q \cup R_{|\overline{A}})
\end{equation}
for all programs $R$, % over $\Lits$, 
allows to capture the atoms that can be disregarded from
the original program and also the context program, so that
the simplified program can reason over the reduced vocabulary while ensuring that the semantics of the original program is preserved w.r.t. projection. %the modified signature. 

It is known that strong simplifications imply \SP-forgetting 
\cite{zgsswkr23}
and the relation between omission abstraction and forgetting has also been studied \cite{zgskr20}. The characterizations for all of the mentioned notions have been established through the SE-models of programs, which characterizes strong equivalence. However until now it remained unclear how these notions come together.

%The relation between omission abstraction and forgetting has also been investigated \cite{zgskr20}.

\begin{table}
{\small
\begin{tabular}{|c|c|c|}
\hline
$A$& $B$& Strong $A$-simplification relative to $B$\\
\hline\hline
$\emptyset$ & $\emptyset$ & equivalence\\
\hline
$\emptyset$ & $\Lits$ &Strong Equivalence \cite{turner2001strong}\\
\hline
$\emptyset$ & $B$ &relativized Strong Equivalence \cite{DBLP:conf/jelia/Woltran04}\\
\hline
$A$ & $\Lits$ &Strong Simplification \cite{zgsswkr23}\\
\hline
$A$ & $\overline{A}$&Strong Persistence \cite{knorr2014preserving}\\
\hline
$A$ & $C$&  $C\subseteq \overline{A}$, relativized Strong Persistence (Section~\ref{sec:relsp})\\
\hline
$A$ & $\emptyset$&Faithful Abstraction \cite{zgskr18}\\
\hline
\end{tabular}
}
\caption{Overview of the full spectrum of the relativized strong simplification notion introduced in this paper.}
\label{fig:list}
\end{table}

In this paper we bridge this gap through a relaxation of the recent simplification notion, where on the context programs we allow for excluding some dedicated atoms: for sets $A,B$ of atoms, we define the notion of \emph{strong $A$-simplification relative to $B$} where \eqref{eq:a4} holds for all programs $R$ over $B$. All of the above mentioned notions such as (relativized) strong equivalence, strong persistence, faithful abstractions and strong simplifications, then become special cases of this novel relativized equivalence notion, of which a summary can be seen in Table~\ref{fig:list}.
Furthermore we show the conditions for relativized simplifiability and observe that the challenging part is for when the context programs do not contain all the atoms to remove/forget. We then show how the desired simplifications can be obtained by an operator that combines projection and a relaxation of \SP-forgetting. 

Our main contributions are thus as follows (i) We propose the novel concept of relativized strong simplification between programs,  provide the necessary and sufficient conditions for testing relativized strong simplifiability, give semantical characterizations of relativized strong simplifications and discuss the full spectrum of this notion; (ii) we introduce a novel forgetting operator which is a combination of projection and a relaxation of SP-forgetting, which we introduce as relativized SP-forgetting; (iii) we conclude with complexity results.

%\todo[inline]{Should we mention the works on "forks"??}

\section{Background}

\subsubsection{Answer Set Programming} An \emph{(extended) logic program (ELP)} is a finite set of \emph{(extended) rules} of form
\begin{align*}
A_1 \vee \dots \vee A_l \leftarrow &A_{l+1},\dots,A_m,\mi{not}\
A_{m+1},\dots,\mi{not}\ A_n,\\
& \mi{not}\ \mi{not}\ A_{n+1},\dots, \mi{not}\ \mi{not} A_k
\end{align*}
\noindent where 
$A_i$  ($1\leq i \leq k$, 
$0\leq l \leq m \leq n \leq k$)
are atoms from a first-order language, and 
$\mi{not}$
is default negation. We also write a rule $r$ as $H(r) \leftarrow B(r)$ or $H(r) \leftarrow B^+(r),$ $\mi{not}\ B^-(r)$, $\mi{not}\ \mi{not}\ B^{--}(r)$. We call $H(r)=\{A_1,\dots,A_l\}$ the \emph{head} of $r$, $B^+(r) = \{A_{l+1}, \dots, A_m\}$ the \emph{positive body}, $B^-(r)=\{A_{m+1},\dots,A_n\}$ the \emph{negative body} and $B^{--}(r)=\{A_{n+1},\dots,A_k\}$ the \emph{double-negated body} of $r$.
If $H(r)\,{=}\,\emptyset$, then $r$ is a \emph{constraint}. 
A rule $r$ is \emph{disjunctive} if $k=n$; if, in addition, $l\leq 1$ then $r$ is \emph{normal}; $r$ is \emph{positive} if $k=m$ and it is a (non-disjunctive) \emph{fact}  if
$B(r)=\emptyset$ and $l \leq 1$; for $H(r)=\emptyset$, we occasionally write $\bot$.
%

%A rule is \emph{Horn}, if it is normal and positive. 
%
%
%An \emph{(extended) logic program} (ELP) is a finite set of rules. 
In the 
%rest of the paper, 
what follows, we focus on propositional programs over a set of atoms from %the 
universe $\Lits$. Programs with variables reduce to their ground versions as usual. Unless stated otherwise the term \emph{program} refers to a 
(propositional) ELP.
%disjunctive logic program.

Let $I\subseteq \Lits$ be an interpretation.
The \emph{GL-reduct} of a program $P$ w.r.t.\ $I$
%$P^I$ is defined as
is given by $P^I=\{H(r) \leftarrow B^+(r) \mid  r \in P, B^-(r)\cap I=\emptyset, B^{--}(r)\subseteq I\}$. 
An interpretation $I$ 
is a \emph{model} of a program $P$ (in symbols $I\models P$) if, for each $r\in P$, $(H(r)\cup B^-(r) )\cap I \neq \emptyset$ or 
$(B^+(r)\cup B^{--}(r))\not\subseteq I$; $I$
is an \emph{answer set}, if it is a minimal model of $P^I$. We denote the set of all answer sets by $\AS(P)$. 
Two programs $P_1,P_2$ are \emph{equivalent}
if $\AS(P_1)=\AS(P_2)$ and \emph{strongly equivalent (SE)}, denoted by
$P_1\equiv P_2$, if $\AS(P_1\cup R)=\AS(P_2\cup R)$ for every $R$ over
$\Lits$.%, and \emph{uniformly equivalent (UE)}, denoted by $P_1 \equiv_u P_2$ %if $\AS(P_1\cup R)=\AS(P_2\cup R)$ for any set of facts $R$ over
%$\Lits$.

An \emph{SE-interpretation} is a pair
$\htm{X,Y}$ such that $X\subseteq Y\subseteq \Lits$; it is
\emph{total}\/ if $X=Y$ and \emph{non-total}\/ otherwise. An
SE-interpretation $\htm{X,Y}$ is an \emph{SE-model} of a program $P$
if $Y\models P$ and $X\models P^Y$. The set of all SE-models of $P$ is denoted by $\HT(P)$. 
%An SE-model $\htm{X,Y}$ of $P$ is called  \emph{UE-model} of $P$ if $X=Y$ or there is no SE-model $\htm{X',Y}$ with $X\subset X'\subset Y$.
%We denote by $\HT^V(P)$ %(resp.\ $\UE^V$) 
%the set of all SE-models %(UE-models) 
%of $P$ over the set of atoms $V \subseteq \Lits$. If $V=\Lits$, we drop the superscript for simplicity. 
Note that a set $Y$ of atoms is an answer set of $P$ if
$\htm{Y,Y}\in\HT(P)$ and 
%$X\subset Y$ such that
no non-total $\htm{X,Y} \in \HT(P)$ exists. 
Two programs $P_1$ and $P_2$ are strongly equivalent iff $\SE(P_1)=\SE(P_2)$ \cite{turner2001strong}. %;
%they are uniformly equivalent iff
%$\UE(P_1)=\UE(P_2)$
%\cite{eiter2003uniform}. 

Lastly, for a set $S\subseteq \Lits$ of atoms, $S_{|A}$ denotes the projection to the atoms in $A$ and $\overline{S}$ is a shorthand for $\Lits \setminus S$. We also use the notion on pairs, i.e.\ $\htm{X,Y}_{|A}=\htm{X_{|A},Y_{|A}}$ and on sets of objects, i.e. ${\cal S}_{|A}= \{ S_{|A} \mid S\in {\cal S}\}$. %We use the phrases ``projecting away $A$" and ``removing $A$" interchangably.

We next summarize the notions needed for our purposes.

\subsubsection{Relativized Equivalence} \citeauthor{DBLP:conf/jelia/Woltran04} \shortcite{DBLP:conf/jelia/Woltran04} relaxed the notion of strong equivalence to have the added programs, $R$, in a specific language $B \subseteq \Lits$. Its semantical characterization requires a generalization of SE-models as follows.

\begin{defn}%[\cite{DBLP:conf/jelia/Woltran04}]
\label{defn:rel-se}
A pair of interpretations $\htm{X , Y}$ is a
(relativized) $B$-SE-interpretation iff either $X = Y$ or $X \subset (Y\cap B)$.
The former
are called total and the latter non-total $B$-SE-interpretations.
Moreover, a $B$-SE-interpretation $\htm{X , Y}$ is a (relativized) $B$-SE-model of a program P iff:
\be[(i)]
\item $Y \models P$;
\item for all $Y' \,{\subset}\, Y$ with $(Y' \,{\cap}\, B) \,{=}\, (Y \,{\cap}\, B)$, $Y'\,{\nmodels}\, P^Y$; and
\item $X\subset Y$ implies existence of a $X'\subseteq Y$ with $X' \cap B=X$, such that $X' \models P^Y$ holds.
\ee
The set of $B$-SE-models of $P$ is given by $\SE^B(P )$.

\end{defn}

Two programs $P_1$ and $P_2$ are strongly equivalent relative to $B$ iff $\SE^B(P_1)=\SE^B(P_2)$.
%\cite{eiter2007semantical}.

\subsubsection{Forgetting}

 We refer to
%\cite{gonccalves2016ultimate,eite-kern-forget-ki-18} 
\cite{eite-kern-forget-ki-18,goncalves_knorr_leite_2023}
for recent
surveys on forgetting, and briefly define \SP-forgetting. %and just shortly summarize the notions needed here.
% for our purposes. 
%Below are two of the properties considered in forgetting that are relevant for our purposes, where 
For a class $F$ of forgetting operators and a class ${\cal C}$ of programs
\begin{myi}
\itemsep=0pt
%\item [\CP] $F$ satisfies \emph{Consequence Persistence} if, for each  $f{\in}F$, $P \,{\in}\, {\cal C}$ and $V \,{\subseteq}\, {\cal A}$, we have $AS(f(P,V))\,{=}\, AS(P)_{|\overline{V}}.$
\item [\SP] 
$F$ satisfies \emph{Strong Persistence} if, for each $f\in F$, $P \in {\cal C}$ and $A\subseteq \Lits$, we have $AS(f(P,A)\cup R) = AS(P \cup R)_{|\overline{A}}$ for all programs  $R \in {\cal C}$ over $\overline{A}$.
%\item [\UP] 
%$F$ satisfies \emph{Uniform Persistence} if, for each $f\in F$, $P \in {\cal C}$ and $A\subseteq \Lits$, we have $AS(f(P,A)\cup R) = AS(P \cup R)_{|\overline{A}}$ for all sets of facts  $R \in {\cal C}$ over $\overline{A}$.
\end{myi}

Here $f(P,A)$ denotes the result of forgetting
about $A$ from $P$. Strong
persistence is %and uniform persistence are 
also considered for a particular forgetting instance $\langle P,A\rangle$,  for $P \in {\cal C}$ and $A\subseteq \Lits$, denoted by
\SP$_{\langle P,A\rangle}$. % and \UP$_{\langle P,A\rangle}$, respectively 
%\cite{gonccalves2016you}. %,gonccalves2019forgetting}.
%Furthermore 
\citeauthor{gonccalves2016you} \shortcite{gonccalves2016you} introduce a criterion $\Omega$ to characterize the instances for which an operator achieving $\SP_{\htm{P,A}}$ is impossible, which has relations with $\overline{A}$-SE-models as shown below.%, though the exact relation between the notions is not easily visible.

\begin{defn}
Let $P$ be a program over $\Lits$ and $A\subseteq \Lits$. An instance $\htm{P,A}$ satisfies criterion $\Omega$ if there exists $Y\subseteq \Lits\setminus A$ such that the set of sets
%$${\cal R}^{Y}_{\htm{P,A}} = \{ R^{Y,A'}_{\htm{P,A}} \mid A' \in \mi{Rel}^Y_{\htm{P,A}}\}$$
\begin{align*}
{\cal R}^{Y}_{\htm{P,A}} \,{=}\, \{\{X \,{\setminus}\, A &\mid \htm{X,Y\cup A'}\in \HT^{\overline{A}}(P)\} \\
&\mid A' \subseteq A, \htm{Y\cup A',Y\cup A'}\,{\in}\, \HT^{\overline{A}}(P) \}
\end{align*}
is non-empty and has no least element. % , where
%\begin{align*}
%R^{Y,A'}_{\htm{P,A}} =\{ &X \setminus A \mid \htm{X,Y\cup A'}\in \HT(P)\}\\
%\mi{Rel}^Y_{\htm{P,A}}=\{ & A' \subseteq A\mid \htm{Y\cup A',Y\cup A'} \in \HT(P) \hbox{ and }\\
%                       &\nexists A'' \subset A' \hbox{ s.t. } \htm{Y\cup A'',Y\cup A'}\in \HT(P)\}.
%\end{align*}
\end{defn}

It is not possible to forget about $A$ from $P$
while satisfying strong persistence exactly when $\htm{P,A}$ satisfies
criterion $\Omega$. 

\citeauthor{gonccalves2016you} \shortcite{gonccalves2016you} also show that the resulting program obtained from forgetting $A$ from program $P$ by applying an operator $f$ from the class $F_{\textup{SP}}$ of \SP-forgetting operators has the SE-models over $\overline{A}$ as $\SE(f(P,A))=\{\htm{X,Y} \mid Y\subseteq \Lits\setminus A \wedge X \in \bigcap \mathcal{R}^Y_{\htm{P,A}}\}$. 

\subsubsection{Abstraction and Simplification} The general notion of abstraction as an over-approximation is defined as follows.

\begin{defn}[\cite{zgskr18}]
\label{def:abs}
For programs $P$ (over $\Lits$) and $Q$ (over $\Lits'$) with $|\Lits|\,{\geq}\,|\Lits'|$,
and a
mapping $m : \Lits \rightarrow \Lits' \cup \{\top\}$, $Q$ is an \emph{abstraction} of $P$ w.r.t.
$m$, if $m(\AS(P))\subseteq \AS(Q)$.
%for every answer set $I$ of $P$,
%$I'=\{m(a) \mid a \in I\}$ is an answer set\/ of $Q$.
\end{defn}

For an \emph{omission abstraction mapping} that omits a set $A$ of atoms from $\Lits$, it becomes $AS(P)_{|\overline{A}}\subseteq AS(Q)$. An abstraction $Q$ is called \emph{faithful} if $AS(P)_{|\overline{A}}=AS(Q)$.

Saribatur and Woltran \shortcite{zgsswkr23} generalized this notion for disjunctive logic programs (DLP) to consider newly added rules or facts that also get abstracted. For that they consider context programs $R$ over $\Lits$ to be \emph{$A$-separated}, which means they are of form $R = R_1 \cup R_2$ for programs $R_1$ and $R_2$ that are defined over $\Lits \setminus A$ and $A$, respectively.

\begin{defn}[\cite{zgsswkr23}]\label{defn:strongsimplification-omit}
Given $A\subseteq \Lits$ and a program  $P$ (over $\Lits$), a program $Q$ (over $\Lits\setminus A$) is \emph{a strong $A$-simplification of $P$} if for any program $R$ over $\Lits$ that is $A$-separated, we have 
\beq
AS(P \cup R)_{|\overline{A}} = AS(Q \cup R_{|\overline{A}})
\eeq{eq:s}
	We say that $P$ is strong $A$-simplifiable if there is a program $Q$ such that (\ref{eq:s}) holds.
\end{defn}

It was shown that the SE-models of $P$ need to satisfy the below conditions, where $A$ is semantically behaving as facts, in order to ensure the existence of such a simplification.

\begin{thm}[\cite{zgsswkr23}]\label{prop:se-deltas}
There exists a strong $A$-simplification of $P$ iff
$P$ satisfies the following
\begin{myi}
\item[$\Delta_{s_1}$:] $\htm{Y,Y} \in \SE(P)$ implies $A\subseteq Y$.
\item[$\Delta_{s_2}$:] 
For any $\htm{X,Y}{\in}\HT(P)$, $X_{|\overline{A}}{=}Y_{|\overline{A}}$ implies $X{=}Y$.
\item[$\Delta_{s_3}$:] 
$\htm{X,Y}{\in} \HT(P)$ implies %for all $\htm{Y',Y'}\in \HT(P)$ with $Y'_{|\overline{A}}=Y_{|\overline{A}}$, 
$\htm{X{\cup} (Y {\cap} A),Y} {\in} \HT(P)$.
\end{myi}
\end{thm}

\nop{
The authors also consider a projection operator defined as follows.

\begin{defn}[\cite{zgsswkr23}]\label{defn:projectaway}
Given a rule $r : H(r) \leftarrow B(r)$, the projection of $r$ onto $\overline{A}$, denoted by $r_{|\overline{A}}$, gives 
$$
\left\{ 
\ba {ll}
\emptyset & \mbox{if } B^-(r)\cap A \,{\neq}\, \emptyset \mbox{ or } H(r) \cap A \,{\neq}\, \emptyset\\ 
H(r) \leftarrow B(r)\,{\setminus}\, A & \mbox{otherwise}. 
\ea
\right.
$$
The resulting program, denoted by $P_{|\overline{A}}$, is then $\bigcup_{r \in P}r_{|\overline{A}}$.
\end{defn}
}

The simplifications are shown to have SE-models equaling $\SE(P)_{|\overline{A}}$.
For strong $A$-simplifiable programs, projecting away the atoms in $A$ achieves the desired simplification.

\begin{thm}[\cite{zgsswkr23}]\label{thm:strong-project}
Let $P$ be a strong $A$-simplifiable program. Then  $P_{|\overline{A}}$ is a strong $A$-simplification of $P$. 
\end{thm}

Here $P_{|\overline{A}}$ refers to removing the atoms in $A$ from the positive bodies of rules, and omitting the rule all together if an atom from $A$ appears in the negative body or the head. 

We will later show that such a projection can still be partially applicable in the relaxation of the simplification notion, while an additional operator more close to forgetting will be needed as well.

\section{Relaxing Strong Simplifications}\label{sec:main}

A natural relaxation for strong simplification is to allow excluding some atoms from the added programs. Thus we propose the following notion.

\begin{defn}\label{defn:strong-rel-omit}
Given $A, B\subseteq \Lits$ and a program  $P$ (over $\Lits$), a program $Q$ (over $\Lits\setminus A$) is \emph{a (strong) $A$-simplification of $P$ relative to $B$} if for any program $R$ over $B$ that is $A$-separated, we have 
\beq
AS(P \cup R)_{|\overline{A}} = AS(Q \cup R_{|\overline{A}})
\eeq{eq:s-rel}
	We say that $P$ is $B$-relativized (strong) $A$-simplifiable if there is a program $Q$ such that (\ref{eq:s-rel}) holds.
\end{defn}

This relaxed notion of strong simplifiability allows to identify programs which are originally not strong simplifiable, but are relativized strong simplifiable when some atoms are not taken into account in the context programs.
%
%For the below result, in addition to the properties Lemma~\ref{lem:rel-se}, we need the additional below restriction on $P$ so that SE-models and $B$-SE-models coincide. 
%\be[($*$)(i)]
%\item If $\htm{Y,Y}\in\SE(P)$, then $\htm{Y,Y}\in \SE^B(P)$.
%\item If $\htm{X,Y}\in\SE(P)$, then $\htm{X\cap B,Y} \in \SE^B(P)$.
%\ee
%
Below are examples of such programs. %which are originally not strong simplifiable when the context program is over the full $\Lits$ (though can become one with further added atoms), but are strong simplifiable relative to a subset of $\Lits$.

\begin{exmp}\label{ex1}
    Let program $P_1$ consist of rules
    \beeq
a \leftarrow b, c. \quad c \leftarrow d. \quad b.  
\eeeq
and program $P_2$ consist of rules
    \beeq
a \leftarrow \mi{not}\ b. \quad b \leftarrow \mi{not}\ a. \quad c. 
\eeeq
$P_1$ and $P_2$ are not strong $\{b,c\}$-simplifiable (though the programs $P_1 \cup \{c.\}$ and $P_2 \cup \{b.\}$ are).

However $P_1$ is strong $\{b,c\}$-simplifiable relative to $\{a,b,d\}$, since the program $Q_1$ consisting of rule 
 \beeq
a \leftarrow d.  
\eeeq
is such a simplification,
and $P_2$ is strong $\{b,c\}$-simplifiable relative to $\{a,c\}$, since the program $Q_2$ consisting of rule 
 \beeq
a \leftarrow \mi{not}\ \mi{not}\ a.% \quad c. 
\eeeq
is such a simplification. 
\end{exmp}

Note that in Definition~\ref{defn:strong-rel-omit} there are no restrictions on how $B$ and $A$ might relate. Thus there can be cases where not all atoms in the set $A$ appear in $R$. We will see that such cases are the cause for the notion of relativized (strong) simplifications being more challenging than strong simplifications.

The context programs that do not contain any atoms from $A$ would be trivially $A$-separated, thus the relativized simplification notion gets reduced to \SP-forgetting. 

\begin{prop}\label{prop:sp}
A forgetting operator $f$ satisfies \SP$_{\langle P,A\rangle}$ iff $f(P,A)$ is an $A$-simplification of $P$ relative to $\overline{A}$.
\end{prop}

\nop{
To observe the case where outside of $B$ there are also atoms not from $A$ consider the below example, an alteration of the one from \cite{DBLP:conf/jelia/Woltran04}.
\begin{exmp}
 Consider two programs 
   \beeq
P=\{a \vee b;\, a \leftarrow b, d;\, b \leftarrow a,d;\, d;\,\leftarrow \mi{not}\ c \}
\eeeq
  \beeq
P'=\{a \leftarrow \mi{not}\ b;\, b \leftarrow \mi{not}\ a;\, a \leftarrow b;\, b \leftarrow a;\, \leftarrow \mi{not}\ c\}
\eeeq
For $\Lits=\{a,b,c,d\}$, $P'$ (over $\Lits\setminus \{d\}$) is a strong $\{d\}$-simplification of $P$ relative to $B$, for any $B$ with $c\notin B$.
%If we consider $P$ and $P'$ to be over $\Lits=\{a,b,c,d\}$, for any $B$ with $c\notin B$, $P$ and $P'$ are strongly equivalent relative to $B$. Furthermore, if we limit $P'$ to be over $\{a,b,c\}$, then $P'$ is a strong $\{d\}$-simplification of $P$ relative to $B$, for any $B$ with $c\notin B$.

\end{exmp}
}

Similar to strong simplifications, not every program might have a relativized simplification. By investigating the undesired case that prevents a program from being relativized simplifiable, which is similar to Proposition~2 from \cite{zgsswkr23}, thus omitted for brevity, 
%
%\nop{
%\begin{prop}\label{prop:s-rel-condition}
%If there is a strong $A$-simplification of $P$ relative to $B$, %If $P$ is uniform $A$-simplifiable, 
%then $P$ satisfies the condition
%\begin{align}
%\forall Y \hbox{(over $\Lits$)} \forall Z \hbox{(over $B$)}, Y \in \AS(P\cup Z),\nonumber\\ \forall Z' \hbox{(over $B$)}, Z'_{|\overline{A}}=Z_{|\overline{A}}\label{eq:s-red-desiredcase}\\
%\exists Y' \hbox{(over $\Lits$)}, Y'_{|\overline{A}}=Y_{|\overline{A}}, Y' \in \AS(P\cup Z')\nonumber %\nonumber
%\end{align}
%\end{prop}
%}
%
%To talk about the relativized case, we bring in the $B$-SE models notion.
we obtain our first result which adjusts the conditions in Theorem~\ref{prop:se-deltas} to the relativized case considering the $B$-SE models.\footnote{For proofs of theorems marked by $\star$ see the appendix.}

\begin{prop}[$\star$]\label{prop:se-rel-deltas}
Let $P$ be a program and $A,B$ be sets of atoms. %Let $P$ satisfy the restrictions ($*$)(i-ii) for $B$. 
If there exists an $A$-simplification of $P$ relative to $B$ then
$P$ satisfies following
\begin{myi}
\item[$\Delta_{s_1}^r$:] $\htm{Y,Y} \in \SE^B(P)$ implies $A\cap B \subseteq Y$.
\item[$\Delta_{s_2}^r$:] 
For any $\htm{X,Y}\in\SE(P)$ with $\htm{Y,Y}\in \SE^B(P)$, $X_{|\overline{A}}=Y_{|\overline{A}}$ implies $X\,{=}\,Y$.
\item[$\Delta_{s_3}^r$:] 
$\htm{X,Y}{\in} \SE^B(P)$ implies 
$\htm{X\cup (Y\cap (A\cap B)),Y} \in \HT^B(P)$.
%\item[$\Delta_{s_4}^B$:] If $\htm{X,Y},\htm{X',Y'} \in \SE^B(P), X\subset Y, X'\subset Y'$ where $Y_{|\overline{A}}=Y'_{|\overline{A}}$ then either $\htm{X,Y'}\in \SE^B(P)$ or $\htm{X',Y} \in \SE^B(P)$. ?????
\end{myi}
\end{prop}

One can see that the restrictive conditions that were required from the SE-models in Theorem~\ref{prop:se-deltas} are relaxed to only hold for the $B$-SE-models, since those are the ones of importance for the answer sets of $P \cup R$ for $R$ over $B$.

We shortly say that $P$ \emph{satisfies $\Delta^r$} if it satisfies the conditions $\Delta_{s_i}^r$ for $1 \leq i \leq 3$.
The following example illustrates checking the $\Delta^r$ conditions.

\begin{exmp}[Ex.~\ref{ex1} ctd]
The SE-models of $P_1$ are
\begin{align*}
 \htm{b c a d ,  b c a d } & \quad \htm{b c a ,  b c a } \quad   \htm{ b a ,  b a }\quad \htm{ b ,  b }\\ 
 \htm{b c a ,  b c a d } &\quad \htm{b a ,  b c a } \quad~~ \htm{ b ,  b a }\\ 
 \htm{b a ,  b c a d } & \quad\htm{ b ,  b c a d } \quad~~\htm{ b ,  b c a }  
   \end{align*}
   For $B=\{a,b,d\}$, $SE^B(P_1)=\{\htm{bcad,bcad},\htm{ba,ba},$ $\htm{b,b},%\htm{bad,bcad},
   \htm{ba,bcad},\htm{b,bcad},\htm{b,ba}\}$. Now for $A=\{b,c\}$, we can easily see that $\Delta_{s_1}^r$ and $\Delta_{s_3}^r$ are  satisfied since each $B$-SE-model contains $A\cap B=\{b\}$, and $\Delta_{s_2}^r$ is trivially satisfied since there is no relevant model. %. $\htm{bcad,bcad} \in SE^B(P_1)$ is the only model with $X'_{|\overline{A}}\,{=}\,Y_{|\overline{A}}$ for $\htm{X',Y}=\htm{bcad,bcad} \in SE(P_1)$ and satisfies $\Delta_{s_2}^r$.
\end{exmp}

%\todo[inline]{by defn of $B$-se models $\htm{bad,bcad}$ shouldnt be a $B$-se model actually, right? }

Observe that, for $B= \Lits$, the conditions $\Delta_{s_i}^r$ become the same with the conditions $\Delta_{s_i}$ of strong simplification, for $1 \leq i \leq 3$. %Thus by using Theorem~\ref{prop:se-deltas} we get the following.
On the other hand,
if $B$ is such that $A\cap B=\emptyset$, i.e.,\ $B\subseteq
\overline{A}$, %we observe that 
the conditions become immaterial.

%\begin{cor}
%    $\Delta_{s_1}^r,\Delta_{s_2}^r$ and $\Delta_{s_3}^r$ are necessary and sufficient for $B$-relativized $A$-simplifiability for when $B=\Lits$.
%\end{cor}
\begin{prop}\label{prop:delta_forget}
Any program $P$ satisfies $\Delta^r$, for any $A,B$ with $B\subseteq \overline{A}$.
\end{prop}
\begin{proof}[Proof (Sketch)]
$\Delta_{s_1}^r$ and $\Delta_{s_3}^r$ trivially holds as $A \cap B=\emptyset$. For some $\htm{X,Y} \in \SE(P)$ to violate $\Delta_{s_2}^r$, $X$ and $Y$ need to differ on the atoms from $A$, while $X\cap \overline{A} = Y \cap \overline{A}$ holds which contradicts $\htm{Y,Y}\in \SE^B(P)$.
\end{proof}

Unsurprisingly, the $\Delta^r$ conditions are not sufficient for $B$-relativized $A$-simplifiability in general. % (for when $B\subset \Lits$). 
This can easily be seen for the case when the context programs do not contain atoms to remove, making use of Proposition~\ref{prop:sp} and the knowledge that
%
%Knowing that 
not every program has a set of atoms which can be forgotten by satisfying \SP. %Next example shows when $\Delta^r$ might not be sufficient for relativized simplifiability.

\begin{exmp}
 Let program $P_3$ consist of rules
\beeq 
a \leftarrow p. \quad b \leftarrow q. \quad p \leftarrow \mi{not}\ q. \quad q \leftarrow \mi{not}\ p.
\eeeq
For $A=\{p,q\}$ and $B=\{a,b\}$, $P_3$ satisfies $\Delta^r$ (Proposition~\ref{prop:delta_forget}), but is not $B$-relativized $A$-simplifiable, since no forgetting operator satisfies $\SP_{\htm{P_3,A}}$ \cite{gonccalves2016you}.
\end{exmp}

%Later, we will show the additional condition needed to ensure relativized simplifiability. 

Note that, when $A \subseteq B$, due to the definition of $B$-SE-models, in order to satisfy $\Delta^r_{s_3}$ there cannot be non-total $\htm{X,Y}\in SE^B(P)$ with $X \subset Y\cap B$. In addition to $\Delta_{s_1}^r$ and $\Delta_{s_2}^r$, these become quite restrictive conditions on the SE-models of $P$. In fact, as we shall see later, the $\Delta^r$ condition turns out to be sufficient for relativized simplifiability when $A\subseteq B$.
Though first we need to understand the semantical characterization of such simplifications.

%However, this changes when there are atoms to be removed which do not appear in $B$, i.e., $A\nsubseteq B$. Then the $B$-SE-models might differ in terms of the atoms from $A\setminus B$ which cannot be distinguished in the projection, where the $\Delta^r$ condition can no longer be sufficient.

\subsection{From $B$-SE-models to $A$-$B$-SE models}
We investigate the semantical characterization of relativized simplifications of a program. %through their $B$-SE-models. 
For that we first introduce the following notion of $A$-$B$-SE-models, which project those $B$-SE-models of importance w.r.t. $A$.%, to make it possible to give a characterization for the relativized simplifications.

\begin{defn}
Given program $P$ over $\Lits$ and $A,B\subseteq \Lits$, the $A$-$B$-SE-models of $P$ are given by the set
\begin{align*}
\SE^B_A(P)=&\{\htm{Y_{|\overline{A}},Y_{|\overline{A}}} \mid \htm{Y,Y} \in \SE^B(P)\} \cup\\ 
&\{\htm{X_{|\overline{A}},Y_{|\overline{A}}} \mid \htm{X,Y} \in \SE^B(P), X \subset Y,\\ &\hbox{and for all } \htm{Y',Y'} \,{\in}\, \SE^B(P) \hbox{ with } Y'_{|\overline{A}}\,{=}\,Y_{|\overline{A}},\\
&\htm{X',Y'} \,{\in}\, \SE^B(P) \hbox{ with } X'_{|\overline{A}}\,{=}\,X_{|\overline{A}}\} 
\end{align*}
\end{defn}

The set of $A$-$B$-SE-models 
collects the projection of all total $B$-SE-models $\htm{Y,Y}$ and all non-total $B$-SE-models $\htm{X,Y}$ for which a respective non-total $B$-SE-model $\htm{X',Y'}$ can be found that agree on the projection, among all total $B$-SE-models $\htm{Y',Y'}$ that agree on the projection with $\htm{Y,Y}$.

\begin{exmp}[Ex.~\ref{ex1} ctd]
For $A{=}\{b,c\}$ and $B{=}\{a,b,d\}$, none of the total $B$-SE-models agree on the projection onto $\overline{A}=\{a,d\}$. So $\SE^B_A$ simply collects the projection of those models and their non-total models. Thus $\SE^B_A(P_1)=\{\htm{ad,ad},\htm{a,a},\htm{\emptyset,\emptyset}\} \cup \{\htm{a,ad},$ $\htm{\emptyset,ad},\htm{\emptyset,a}\}$.

Now assume that another program $P_1'$ has the SE-models $\SE(P_1')=\SE(P_1)\setminus \htm{ba,bca}$. Then $\htm{bca,bca}$ is added to $\SE^B(P_1')$ in addition to $\SE^B(P_1)$. Since $\htm{bca,bca}_{|\{a,d\}}=\htm{ba,ba}_{|\{a,d\}}=\htm{a,a}$, in order for $\htm{\emptyset,a}$ to be in $\SE^B_A(P_1')$ there needs to be some non-total $B$-SE-model of form $\htm{.,bca}$ that can be projected onto $\htm{\emptyset,a}$ which is not the case. Thus $\SE^B_A(P_1')=\SE^B_A(P_1) \setminus \htm{\emptyset,a}$.
\end{exmp}

\nop{
Clearly we have the following.
\begin{lemma}
    $\SE^B_A(P) \subseteq \SE^B(P)_{|\overline{A}}$
\end{lemma}
}

%We next investigate the characterization of the (relativized) SE-models of the simplifications. 
For referring to the relativized SE-models of the simplifications in our next result, let us introduce a notation for the set of SE-models of a program over $A_1$ relativized to $A_2$.

\begin{defn}
Let $P$ be a program. The relativization of SE-models of $P$ over $A_1$ to the set $A_2$ of atoms is denoted by
$$\SE^{A_1,A_2}(P)=\{\htm{X,Y} \mid \htm{X,Y} \in \SE^{A_2}(P), Y \subseteq A_1\}.$$
\end{defn}

Interestingly, the relativized SE-models of any $A$-simplification for a program $P$ relative to $B$, if exists, need to adhere with the  $A$-$B$-SE-models of $P$. 

\begin{prop}[$\star$]\label{prop:se-q-rel-1}
%Let $P$ satisfy $\Delta_{u_2}$. 
If $Q$ is an $A$-simplification for $P$ relative to $B$, then it satisfies
\beq
SE^B_A(P)=SE^{\overline{A},B\setminus A}(Q).
\eeq{eq:se_rel_proj}
\end{prop}

\begin{exmp}[Ex.~\ref{ex1} ctd]
Equation \eqref{eq:se_rel_proj} holds for $P_1$ and $Q_1$.
%We get $SE^{\overline{A},B\setminus A}(Q_1)=SE^{B}_{\overline{A}}(P_1)$.
\end{exmp}
%Even though $\SE^B(P)$ satisfies Lemma~\ref{lem:rel-se-yy-xx}, it could be that $\SE^B(P)_{|\overline{A}}$ does not. This can occur when for $\htm{Y,Y},\htm{X,X}\in \SE^B(P)_{|\overline{A}}$, $X\subset Y$ does not hold, while $X_{|\overline{A}} \subset Y_{|\overline{A}}$. If $P$ satisfies $\Delta_{s_1}^r, \Delta_{s_2}^r$ and $\Delta_{s_3}^r$, this can only occur if $A\cap B=\emptyset$.

\nop{
As a final remark, remember the discussion on different total $B$-SE-models that agree on the projection occurring only when $A\nsubseteq B$. Thus when $A \subseteq B$, the $A$-$B$-SE-models, in fact, simply collects the projection of the $B$-SE-models.

\begin{prop}
Given program $P$ over $\Lits$ and $A,B\subseteq \Lits$ such that $A\subseteq B$, it holds that $\SE^B_A(P)=\SE^B(P)_{|\overline{A}}$.
\end{prop}
}

Now we can show the sufficiency of $\Delta^r$ for when $A\subseteq B$.

\begin{thm}
Given program $P$ over $\Lits$ and $A,B \subseteq \Lits$ such that $A \,{\subseteq}\, B$. $P$ satisfies $\Delta^r$ iff there exists an $A$-simplification of $P$ relative to $B$.
\end{thm}
\begin{proof}[Proof (Sketch)]
When $A {\subseteq} B$, due to $\Delta_{s_1}^r$, there cannot be different total $B$-SE-models agreeing on the projection, %since they all contain $A$. 
thus $\SE^B_A(P)$ %simply 
amounts to $\SE^B\!(P)_{|\overline{A}}$. %, i.e., $\SE^B_A(P)=\SE^B(P)_{|\overline{A}}$. 
Due to Proposition~\ref{prop:se-rel-deltas}, what remains is to show that a program with $(B\!\setminus\! A)$-SE-models matching $\SE^B(P)_{|\overline{A}}$ is %in fact 
a relativized simplification of $P$, which follows a very 
similar proof to that of Theorem~\ref{prop:se-deltas}.
\end{proof}

This result shows that the challenge of relativized simplifiability is in fact due to having atoms to be removed that do not appear in the context programs. Then the $B$-SE-models might differ in terms of the atoms from $A\setminus B$ which cannot be distinguished in the projection, making the $\Delta^r$ condition no longer sufficient. Thus in order to characterize relativized simplifiability in general, an additionaly property is needed.

\nop{
Remark that, if $P$ satisfies $\Delta_{s_1}^r$, the case for there to be different total $B$-SE-models which agree on the projection occurs only when $A \nsubseteq B$, since then the atoms from $A\setminus B$ in the total models may differ, which eventually gets projected away.
%
%\begin{lemma}
%If $P$ satisfies $\Delta_{s_1}^r$, for $\htm{Y,Y},\htm{Y',Y'} \in \SE^B(P)$, $Y'_{|\overline{A}}=Y_{|\overline{A}}$ while $Y'\neq Y$ only if $B\subseteq \overline{A}$. %For a DLP $P$ the right hand side becomes $B\subseteq \overline{A}$.
%\end{lemma}
%
This is the reason that makes the notion of relativized simplifications more challenging, since this case %above 
does not appear when the context program is over $\Lits$. Thus in order to characterize relativized simplifiability, an additional property is needed that handles the case of having atoms to be removed that do not appear in the context programs.
}

\subsection{Characterizing relativized simplifiability}

%Now we move on to.
We introduce the following criterion that will help us in %which plays the fundamental role in  
obtaining the sufficient conditions.% for relativized simplifiability. %obtaining the sufficient conditions for when the simplifiability is relaxed to the relativization.

\begin{defn}
Let $P$ be a program over $\Lits$ and $A,B\subseteq \Lits$. $P$ satisfies criterion $\Omega_{A,B}$ if there exists $Y\subseteq \Lits\setminus A$ such that the set of sets
%$${\cal R}^{Y}_{\htm{P,A,B}} = \{ R^{Y,A'}_{\htm{P,A,B}} \mid A' \in \mi{Rel}^Y_{\htm{P,A,B}}\}$$
\begin{align*}
{\cal R}^{Y}_{\htm{P,A,B}} \,{=}\, \{\{X \,{\setminus}\, A &\mid \htm{X,Y\cup A'}\in \HT^{B}(P)\} \\
&\mid A' \subseteq A, \htm{Y\cup A',Y\cup A'}\,{\in}\, \HT^{B}(P) \}
\end{align*}
is non-empty and has no least element.%, where
%\begin{align*}
%R^{Y,A'}_{\htm{P,A,B}} {=}\,\{ & Y \mid \htm{Y\cup A',Y\cup A'}\in \HT(P)\} \cup\\
%\{ &(X \,{\cap\,} B) \,{\setminus}\, A \mid \htm{X,Y\,{\cup}\, A'}\,{\in}\, \HT(P), X\,{\subset}\, Y\cup A'\}\\\underline{}
%\mi{Rel}^Y_{\htm{P,A,B}}{=}\,\{ & A' \subseteq A\mid \htm{Y\cup A',Y\cup A'} \in \HT(P) \hbox{ and }\\
%                       &\nexists Y' \subset Y \cup A' \hbox{ s.t. } Y'\cap B = (Y\cup A') \cap B \hbox{ and }\\& \htm{Y',Y\cup A'}\in \HT(P)\}.
%\end{align*}
\end{defn}

The difference of $\Omega_{A,B}$  from $\Omega$ is that %the relativization onto a set $B$ of atoms needs to be taken into account, 
instead of $\overline{A}$-SE-models, now $B$-SE-models are taken into account. In fact, ${\cal R}^{Y}_{\htm{P,A,\overline{A}}}={\cal R}^{Y}_{\htm{P,A}}$. Thus we have the following.%which is trivially resolved in \SP-forgetting when $B=\overline{A}$. %when $B$ is set to $\overline{A}$, criterion $\Omega_{A,B}$ amounts to criterion $\Omega$ from \cite{gonccalves2016you}. Thus we get the following.

\begin{prop}\label{prop:omegas}
    $P$ does not satisfy $\Omega_{A,\overline{A}}$ iff  $\langle P,A\rangle$ does not satisfy $\Omega$. 
\end{prop}

We illustrate how the criterion $\Omega_{A,B}$ can be checked with the following example.

\begin{exmp}
Consider a program $P$ with SE-models
\begin{align*}
\langle a b c ,  a b c  \rangle & \quad\langle  a b d ,  a b d  \rangle \quad\langle a b c d ,  a b c d  \rangle\\
\langle a c ,  a b c  \rangle & \quad\langle  b d ,  a b d  \rangle \quad~~\langle a b c ,  a b c d  \rangle\\
\langle a c ,  a b c d  \rangle &\quad\langle b d ,  a b c d  \rangle\quad
\langle a b d ,  a b c d  \rangle
\end{align*}
Let $A=\{c,d\}$ and $B=\{a,b,c\}$. For $Y=\{a,b\}$, only $\htm{abc,abc}$ and $\htm{abd,abd}$ appear in $\SE^B(P)$. %satisfy the conditions for $\mi{Rel}^{\{a,b\}}_{\htm{P,A,B}}$. 
%Thus we have
%we have $\SE^B(P)=\{\langle a b c ,  a b c  \rangle, \langle a c ,  a b c  \rangle , \langle  a b d ,  a b d  \rangle, \langle  b ,  a b d  \rangle\}$, while $\SE^B_A(P)=\{\langle abc,abc\rangle,\langle abd,abd\rangle\}$, thus criterion  $\Omega_{A,B}$ is not satisfied. 
%$\mi{Rel}^{\{a,b\}}_{\htm{P,A,B}}=\{\{c\},\{d\}\}$. 
There are $\htm{abc,abc}, \htm{ac,abc} \in \SE^B(P)$ of form $\htm{X,abc}$, %so $R^{\{a,b\},\{c\}}_{\htm{P,A,B}}=\{\{a,b\},\{a\}\}$. Similarly for models of form $\htm{X,abd}$, we get
and $\htm{abd,abd}, \htm{bd,abd} \in \SE^B(P)$ of form $\htm{X,abd}$.
%$R^{\{a,b\},\{d\}}_{\htm{P,A,B}}=\{\{a,b\},\{b\}\}$. 
Thus ${\cal R}^{\{a,b\}}_{\htm{P,A,B}}=\{\{\{a,b\},\{a\}\},\{\{a,b\},\{b\}\}\}$ is non-empty and does not have a least element. So $\Omega_{A,B}$ is satisfied. 
\end{exmp}

As the set $\mathcal{R}^Y_{\htm{P,A,B}}$ collects all $B$-SE-models that can be projected onto $\htm{\cdot,Y}$, the criterion $\Omega_{A,B}$ can in fact be characterized through $A$-$B$-SE-models.
%when for all $Y$, $\mathcal{R}^Y_{\htm{P,A,B}}$ is either empty or has a least element, its intersection amounts to the $A$-$B$-SE-models of $P$.

%Our notion of $A$-$B$-SE-models makes it possible to easily relate the satisfaction of criterion $\Omega_{A,B}$ with the $B$-SE-models of the program. The following observation shows how the criterion $\Omega_{A,B}$ manages to obtain the $B$-SE-models of importance.

\begin{prop}\label{prop:omega-a-b-se}
$P$ does not satisfy criterion $\Omega_{A,B}$ iff $\{\htm{X,Y} \mid Y\subseteq \Lits\setminus A \wedge X \in \bigcap \mathcal{R}^Y_{\htm{P,A,B}}\} = SE^B_A(P)$.
\end{prop}
\begin{proof}[Proof (Sketch)]
If criterion $\Omega_{A,B}$ is not satisfied, making use of $\bigcap \mathcal{R}^Y_{\htm{P,A,B}}$ to gather all projected models $\htm{X,Y}$, which have some respective $B$-SE-model $\htm{X',Y'}$ with $X'_{|\overline{A}}\,{=}\,X$ for any $B$-SE-model $\htm{Y',Y'}$ with $Y'_{|\overline{A}}\,{=}\,Y$, reaches $\SE^B_A(P)$. If $\Omega_{A,B}$ is satisfied, the collection will return the empty set, while $\SE^B_A(P)$ is not.
\end{proof}
\nop{
So whenever $P$ satisfies the criterion $\Omega_{A,B}$, there are $\htm{Y \cup A',Y\cup A'},\htm{Y \cup A'',Y\cup A''} \in \SE(P)$ that are also in $\SE^B(P)$, but for a non-total model $\htm{X,Y\cup A'} \in \SE^B(P)$ there is no non-total model $\htm{X',Y\cup A''}$ in $SE^B(P)$ that agree on the projection to $\overline{A}$, i.e., $X'_{|\overline{A}}=X_{|\overline{A}}$.
In other words, we have the following.

\begin{cor}
    A program $P$ does not satisfy criterion $\Omega_{A,B}$ iff for any $\htm{Y,Y} \in SE^B_A(P)$ there is some $\htm{Y\cup A',Y\cup A'} \in SE^B(P)$ such that $\{\htm{X,Y} \mid \htm{X\cup A'',Y\cup A'} \in SE^B(P)\}=\{\htm{X,Y} \mid \htm{X,Y} \in SE^B_A(P)\}$.
\end{cor}
}

%\begin{exmp}
%\todo[inline]{ex}
%\end{exmp}

It turns out that criterion  $\Omega_{A,B}$ is also necessary for relativized simplifiability.

\begin{prop}[$\star$]\label{prop:criterion}
If $P$ satisfies criterion $\Omega_{A,B}$, then a $B$-relativized $A$-simplification cannot exist.
\end{prop}

%\begin{exmp}
%\todo[inline]{ex}
%\end{exmp}

%Observe that for a program $P$ and a set $B$ of atoms if $B$-SE-models of $P$ is the empty set then for any $A$, for all $Y\subseteq \Lits \setminus A$, $\mathcal{R}^Y_{\htm{P,A,B}}=\emptyset$ holds, making $\Omega_{A,B}$ not satisfied. However, for the case of $A=\Lits$ this causes issues in finding a relativized simplification. %this means that then $\{\htm{X,Y} \mid Y\subseteq \Lits\setminus A \wedge X \in \bigcap \mathcal{R}^Y_{\htm{P,A,B}}\}$ would be the empty set, which may then cause the issue of not finding a relativized simplification satisfying \eqref{eq:se_rel_proj}.

%\begin{exmp}
%Consider the program $P_4$ (over $\{a,b,c\}$) with $\SE(P_4)=\{\htm{abc,abc},\htm{a,abc}\}$. Let $A=\{a,b,c\}$ and $B=\{a\}$. $P_4$ satisfies $\Delta^r$ since $\SE^B(P_4)=\emptyset$ and does not satisfy $\Omega_{A,B}$. The only candidate for an $A$-simplification of $P_4$ relative to $B$ is the empty program. However $AS(P_4\cup \{a\}\neq AS(\emptyset \cup \emptyset)$.
%\end{exmp}

%Keeping this special case in mind, we can now provide the sufficient conditions for relativized simplifiability.

We next show that the $\Delta^r$ conditions together with the $\Omega_{A,B}$ criterion are sufficient for relativized simplifiability.

\begin{thm}[$\star$]\label{thm:se-q-rel}
If $P$ satisfies $\Delta^r$ and does not satisfy criterion $\Omega_{A,B}$ then $P$ is $B$-relativized $A$-simplifiable.
\end{thm}

The result is obtained by proving that a program $Q$ satisfying \eqref{eq:se_rel_proj}
is a relativized simplification of $P$.

Thus we have the necessary and sufficient conditions on $B$-relativized $A$-simplifiability, while also providing a characterization %on the SE-models 
for such simplifications, when exist.

\begin{cor}\label{cor:main}
$P$ satisfies $\Delta^r$ and does not satisfy criterion $\Omega_{A,B}$ iff $P$ is $B$-relativized $A$-simplifiable.% $Q$ with $SE^B_A(P)=SE^{\overline{A},B}(Q)$.
\end{cor}

%\begin{exmp}
%\todo[inline]{ex}
%\end{exmp}

\subsection{The Full Spectrum}% of Relativized Simplifications}

We now discuss how the notion of relativized simplification captures all the related notions from the literature (shown in Table~\ref{fig:list}), by looking at the special cases of removing atoms and relativization.

 When $A\,{=}\,\emptyset$, the $\Delta^r$ conditions are trivially satisfied, and as $SE^B_{\emptyset}(P)\,{=}\,SE^B(P)$, the criterion $\Omega_{\emptyset,B}$ is not satisfied for any $B$. Also any program is $\emptyset$-separated. Then $B$-relativized $\emptyset$-simplification simply amounts to relativized strong equivalence \cite{DBLP:conf/jelia/Woltran04}. If, in addition, we set $B\,{=}\, \Lits$, we reach strong equivalence. 

\begin{prop} 
%The following holds.
%\bi
%\item 
$P$ and $Q$ are strong equivalent relative to $B$ iff $Q$ is a $\emptyset$-simplification relative to $B$.
%\item 
Every program is $B$-relativized $\emptyset$-simplifiable, for any $B$.
%\ei
\end{prop}

For $B=\emptyset$, the $\emptyset$-SE models only contain the total models of the answer sets, which captures the notion of faithful omission abstraction. %(furthermore $\htm{Y,Y}$ is an $\emptyset$-SE model iff $Y$ is an answer set of $P$)
Moreover $\Delta^r$ is trivially satisfied and  $\Omega_{A,\emptyset}$ is not satisfied for any $A$.
%Observe that since $\emptyset$-SE models of a program only contains the total models of the answer sets, criterion $\Omega_{A,\emptyset}$ is always not satisfied for any $A$. 

\begin{prop} \label{prop:faithful}
%The following holds.
%\bi
%\item 
$Q$ is an $A$-simplification of $P$ relative to $\emptyset$ iff $Q$ is a faithful abstraction of $P$ for omission of $A$.
%\item    
Every program is $\emptyset$-relativized $A$-simplifiable, for any $A$.
%\ei
\end{prop}

When we additionally set $A$ to $\emptyset$, then we reach equivalence of two programs.

Setting $A=\Lits$ results in omitting all the atoms from the program. Thus as potential relativized simplifications over $\emptyset$ we have either $Q=\emptyset$ or $Q'=\{\bot \leftarrow.\}$. From above we know that every program is $\emptyset$-relativized $\Lits$-simplifiable. So a satisfiable program has $Q$ as its relativized simplification, while an unsatisfiable program has $Q'$. Though for $B\neq \emptyset$, relativized simplifiability might not always hold.

%Actually any $\emptyset$-relativized strong $A$-simplification $Q$ would satisfy
%$$AS(P)_{|\overline{A}} = AS(Q)$$
%thus is also referred as a \emph{faithful abstraction} in the literature.

%\begin{prop}
%The following holds.
%\begin{itemize}
%    \item Any program is $\emptyset$-relativized strong $\emptyset$-simplifiable.
%    \item $Q$ is strong equivalent to $P$ iff $Q$ is an $\emptyset$-relativized
%strong $\emptyset$-simplification of P .
%\end{itemize} 
%\end{prop}

\nop{
Note that we do not give restrictions on $Q$.

\begin{prop}
Let $f$ be a forgetting operator on $\Lits$. For a program $P$ and a set $A$ of atoms, $f(P,A)$ is an $A$-simplification of $P$ relative to $\overline{A}$ iff $f$ satisfies $\SP_{\htm{P,A}}$.
\end{prop}

We see that when $B=\overline{A}$, $A \cap B=\emptyset$ thus $\Delta_{s_1}^r$ and $\Delta_{s_3}^r$ are trivially satisfied. Item (ii) of Definition~\ref{defn:rel-se} ensures that $\Delta_{s_2}^{\overline{A}}$ is satisfied. Thus we conclude the following. 

\begin{prop}
There is a forgetting operator satisfying $\SP_{\htm{P,A}}$ for every (DLP) $P$ and $A$ \textcolor{red}{whenever $SE^B(P)_{|\overline{A}}$ satisfies the DLP conditions}.
\end{prop}
}

When the context programs are set to be over the remaining atoms, i.e., $B=\overline{A}$, we reach $\SP$-forgetting. In the next section we introduce a relativization of $\SP$-forgetting to consider the case of $B \subseteq \overline{A}$, which will be needed in defining operators that obtain the relativized simplifications.

\nop{
\mycomment{Maybe omit the below discussion, or move to somewhere else}
Remark that in \cite{zgsswkr23} any DLP has a strong simplification that is also a DLP. This is obtained due to having the context programs over the whole vocabulary, thus also containing all of the atoms $A$ to be removed. However, when lifting the notions to the relativized case, when $A \nsubseteq B$, not every DLP might have a $B$-relativized strong $A$-simplification that is also a DLP. The reasoning is similar to the result from \cite{gonccalves2016you} on the fact that there is no forgetting operator over the class of DLP that satisfies $\SP$.
}

\section{Combination of Projection and Forgetting}\label{sec:op}

\nop{
First let us show that the concept of $A$-$B$-SE-models also achieves the set of SE-models for the when $B=\overline{A}$.

\begin{prop}
$f \in F_{\textup{SP}}$ iff $\SE(f(P,A))=\SE^{\overline{A}}_A(P)$.
\end{prop}

Remark that in the above result, we make use of the fact that $\SE^{\overline{A},\overline{A}}(f(P,A))$ is the same as $\SE(f(P,A))$ over $\overline{A}$. %Thus equation \eqref{eq:se_rel_proj} becomes an equality among the SE-models of the simplification and $\SE^{B}_A(P)$. 
In Section~\ref{sec:relsp}, we additionally consider the possibility of $B\subseteq \overline{A}$ which requires a new forgetting notion.
}

In this section we introduce an operator that can achieve relativized simplifications. As we know, whenever a relativized simplification exists, it satisfies the equation \eqref{eq:se_rel_proj}. Following the notation from forgetting operators, we introduce a class of operators that achieves these simplifications.  
For this, instead of a forgetting instance $\htm{P,A}$ we consider \emph{relativized forgetting instance} $\htm{P,A,B}$.

\begin{defn}
Let $F_{\textup{rSS}}$ be the class of forgetting operators defined by the following set:
$\{f \mid \SE^{\overline{A},B\setminus A}(f(P,A,B))=\{\htm{X,Y} \mid Y\subseteq \Lits\setminus A \wedge X \in \bigcap \mathcal{R}^Y_{\htm{P,A,B}}\}\}$.
\end{defn}

Proposition~\ref{prop:omega-a-b-se} and Corollary~\ref{cor:main} leads to the following.

\begin{cor}\label{cor:rss}
For $f {\in} F_{\textup{rSS}}$, $f(P,A,B)$ is a $B$-relativized $A$- simplification of $P$ iff $P$ is $B$-relativized $A$-simplifiable.
\end{cor}

Note that above class of operators is similar to $F_{\textup{SP}}$, where instead of $\mathcal{R}^Y_{\htm{P,A}}$ we focus on $\mathcal{R}^Y_{\htm{P,A,B}}$, and instead of giving the characterization over the SE-models of the resulting program, we consider its $B$-SE models.

Clearly when $B=\overline{A}$, since ${\cal R}^{Y}_{\htm{P,A,\overline{A}}}={\cal R}^{Y}_{\htm{P,A}}$, the resulting program after applying an %forgetting 
operator in $F_{\textup{rSS}}$ is %will be 
strongly equivalent to the result of an \SP-forgetting operator.

\begin{prop}
Let $P$ be a program, $A \subseteq \Lits$, $f_{SP}\in F_{\textup{SP}}$, $f_{rSS} \in F_{\textup{rSS}}$. Then 
$f_{SP}(P,A) \equiv f_{rSS}(P,A,\overline{A})$.
\end{prop}

Thus any operator in $F_{\textup{rSS}}$ can be applied as an \SP-forgetting operator. As we shall show next, the forgetting operators in $F_{\textup{rSS}}$ can also be used to achieve a relaxed notion of \SP-forgetting.

\subsection{Relativized Strong Persistence}\label{sec:relsp}

We introduce a relaxed \SP-forgetting notion, where non-forgotten atoms can be excluded from the context program.

\begin{defn}
A forgetting operator $f$ satisfies \emph{relativized strong persistence for a relativized forgetting instance $\htm{P,A,S}$, $S \,{\subseteq}\, \overline{A}$}, denoted by \rSP$_{\htm{P,A,S}}$,  if for all programs  $R$ over $S$, $AS(f(P,A,S)\,{\cup}\, R) \,{=}\, AS(P \,{\cup}\, R)_{|\overline{A}}$ .
%
%We say a class $F$ of forgetting operators satisfies relativized strong persistence, denoted by \rSP, if for each $f\in F$, $f$ satisfies \rSP$_{\htm{P,A,S}}$, for each $P\in \mathcal{C}, A\subseteq \Lits$ and $S \subseteq \overline{A}$.
\end{defn}

Above definition naturally leads to the following.
\begin{prop}
If a forgetting operator $f$ satisfies \SP$_{\htm{P,A}}$ then it satisfies \rSP$_{\htm{P,A,S}}$, for any $S\subseteq \overline{A}$.
\end{prop}

%Thus any forgetting operator from $F_{\textup{SP}}$ can be used as a \rSP-forgetting operator. 
%
%However it is also possible to define a specific class of forgetting operators for \rSP, similar to $F_{\textup{SP}}$, where instead of $\mathcal{R}^Y_{\htm{P,A}}$ we focus on $\mathcal{R}^Y_{\htm{P,A,S}}$.
%
%Clearly we have the following.
%
%\begin{prop}
%$F_{\textup{SP}} \subseteq F_{\textup{rSP}}$.
%\end{prop}
%
In fact, every operator in $F_{\textup{rSS}}$ satisfies \rSP$_{\htm{P,A,S}}$, when possible, which we get by Proposition~\ref{prop:delta_forget} and Corollary~\ref{cor:rss}. 

\begin{thm}
Every $f \in F_{\textup{rSS}}$ satisfies \rSP$_{\htm{P,A,S}}, S\subseteq \overline{A}$, for every relativized forgetting instance $\htm{P,A,S}$, where $P$ does not satisfy $\Omega_{A,S}$. 
\end{thm}

%\begin{proof}
%\todo[inline]{proof}
%\end{proof}

\subsection{An Operator that Projects and Forgets}

We begin with showing that as long as the $\Delta^r$ conditions are satisfied for those of $A$ which appear in $B$ it is possible to project them away\footnote{When projecting from ELPs, the atoms in $A$ are removed from the double negated body of the rules as well.} from $P$ while preserving the semantics. 

We observe that the relativized SE-models of the resulting program after the projecting away atoms occurring in $A\cap B$ equals the $B$-SE-models of $P$ projected onto the remaining atoms. This then also equals its $(A\cap B)$-$B$-SE-models.

\begin{prop}[$\star$]\label{prop:b-se-proj}
Let $P$  satisfy $\Delta^r$ for $A,B \subseteq \Lits$, and $A'=A\cap B$. Then it holds that
%\bi
%\item $P$ is $B$-relativized $A'$-simplifiable.
%\item 
$$\SE^{\overline{A'},B\setminus A}(P_{|\overline{A'}})=\SE^B(P)_{|\overline{A'}}=\SE^B_{A'}(P).$$
%\ei 
\end{prop}

This observation leads to the following result.

\begin{cor}
Let $P$  satisfy $\Delta^r$ for $A,B \subseteq \Lits$, and $A'=A\cap B$. $P_{|\overline{A'}}$ is an $A'$-simplification of $P$ relative to $B$.
\end{cor}

This result shows that whenever $P$  satisfies $\Delta^r$, even if $\Omega_{A,B}$ criterion is satisfied, preventing $B$-relativized $A$-simplifiability, it is still possible to project away atoms in $B\cap A$ to reach a program with a reduced signature.

\nop{
Next we show that projection of the atoms in $B\cap A$ preserves relativized simplifiability.

\begin{prop}\label{prop:rel-proj}
    If $P$ is $B$-relativized $A$-simplifiable, then $P_{|\overline{A \cap B}}$ %(over $\Lits\setminus(A\cap B)$) 
    is $(B {\setminus} A)$-relativized $(A{\setminus} B)$-simplifiable.
    \end{prop}
    \begin{proof}[Proof (Sketch)]
Since $\Delta^r$ is satisfied when $B\,{\cap}\, A\,{=}\,\emptyset$, we only need to show $P'=P_{|\overline{A \cap B}}$ violates $\Omega_{A\setminus B, B\setminus A}$.  %We know that for all $Y$, $\mathcal{R}^Y_{\htm{P,A,B}}$ collects all $B$-SE-models that can be projected onto $\htm{\cdot,Y}$, and it is either empty or has a least element. 
The set $\mathcal{R}^Y_{\htm{P',A\setminus B,B\setminus A}}$ collects all $(B\,{\setminus}\, A)$-SE-models that can be projected onto $\htm{\cdot,Y}$. By Proposition~\ref{prop:b-se-proj}, the $(B\setminus A)$-SE-models of $P'$  amount to $\SE^B(P)_{|\overline{A\cap B}}$, thus $\mathcal{R}^Y_{\htm{P',A\setminus B,B\setminus A}}$ can either be empty or have a least element, since same holds for $\mathcal{R}^Y_{\htm{P,A,B}}$.
%
%For that we look at $\SE^{B\setminus A}_{A\setminus B}(P)_{|\overline{A\cap B}}$ which projects $\SE^{\overline{A\cap B},B\setminus A}(P)_{|\overline{A\cap B}}$.
%
%Since $P$ satisfies $\Omega_{A,B}$ we know that $\SE^B_A(P)$ is either empty By Proposition~\ref{prop:b-se-proj} it is sufficient to look at $\SE^B(P)_{|\overline{A\cap B}}$
\end{proof}
}

Interestingly, if a program is $B$-relativized $A$-simplifiable, obtaining the desired simplification is possible by first syntactically projecting away those atoms in $B\cap A$ and then applying an operator from $F_{\textup{rSS}}$ for \rSP-forgetting that forgets those atoms remaining outside of the context programs. %, $A\setminus B$.

\begin{thm}\label{thm:proj-rss}
    Let $P$ be $B$-relativized $A$-simplifiable, and $f \in F_{\textup{rSS}}$. Then $f(P_{|\overline{A \cap B}},A \setminus B,B\setminus A)$ is a $B$-relativized $A$-simplification of $P$. 
\end{thm}
\begin{proof}[Proof (Sketch)]
By Proposition~\ref{prop:b-se-proj}, the $(B\setminus A)$-SE-models of $P'{=}P_{|\overline{A \cap B}}$  amount to $\SE^B(P)_{|\overline{A\cap B}}$. Thus for any $Y$, $\bigcap\mathcal{R}^Y_{\htm{P',A\setminus B,B\setminus A}} = \bigcap \mathcal{R}^Y_{\htm{P,A,B}}$, which means that if $f$ achieves a $B$-relativized $A$-simplification of $P$ with $f(P,A,B)$, then it can achieve such a simplification with $f(P_{|\overline{A \cap B}},A \setminus B,B\setminus A)$ as well.
\end{proof}

We see that the challenging part of obtaining a relativized simplification when there are atoms to remove that do not appear in the context programs brings us closer to \SP-forgetting. In order to obtain a fully syntactic operator, %Obtaining relativized simplifications syntactically can be achieved by first projecting away those atoms in $A\cap B$ and then making use of an . 
an interesting follow-up work would be to see whether the existing syntactic \SP-forgetting operators \cite{berthold2019syntactic,berthold2022syntactic} can be adjusted to consider \rSP.

%\begin{defn}
%Let $F_{\textup{rSS}}$ be the class of removal operators defined by the following set:
%$\{f \mid \SE^B(f(P,A,S))=\{\htm{X,Y} \mid Y\subseteq \Lits\setminus A \wedge X \in \bigcap \mathcal{R}^Y_{\htm{P,A,S}}\}\}$
%\end{defn}

\section{Computational Complexity}\label{sec:comp}

We provide the complexity of deciding simplifiability through checking the $\Delta^r$ and $\Omega_{A,B}$ conditions, and simplification testing.  We assume familiarity with basic concepts of complexity theory. For comprehensive details we refer to \cite{papadimitriou2003computational,arora2009computational}.

%\begin{prop}{(\cite{eiter2007semantical}, Theorem 6.12.)} 
%Given a program $P$, an SE-interpretation $\htm{X,Y}$ , and $B \subseteq \Lits$ , deciding whether $\htm{X , Y} \in \SE^B( P )$ is $D^P$-complete.
%\end{prop}

We begin with %the observation on 
checking the $\Delta^r$ conditions.

\begin{prop}\label{prop:compdelta}
Let $P$ be a program over $\Lits$ and $A,B\subseteq \Lits$. Deciding whether $P$ satisfies $\Delta^r$ is in $\Pi^P_2$.
\end{prop}
\noindent  Violation of any $\Delta^r_{s_i}$ can be checked in $\Sigma^P_2$ since 
$B$-SE-model checking is in $D^P$ \cite{eiter2007semantical}.

Next we move on to checking the $\Omega_{A,B}$ criterion. For this we follow the results from \cite{DBLP:journals/ai/GoncalvesKLW20}, with the condition that the given program satisfies $\Delta^r$. Remember that for the case of \SP-forgetting, $\Delta^r$ is trivially satisfied.
For the below two results, we make use of $\Delta^r_{s_1}$ and $\Delta^r_{s_3}$, which gives us that whenever $\htm{Y,Y} \in \SE^B(P)$ some $\htm{X,Y} \in \SE^B(P)$ exists iff $\htm{X\cup (A\cap B),Y} \in \SE^B(P)$.

%% SW: this remake was not clear to me.
%We first provide the complexity of determining whether some $X$ occurs in the intersection of  $\mathcal{R}^Y_{\htm{P,A,B}}$ used in the definition of $\Omega_{A,B}$ and $F_{rSS}$.

\begin{prop}\label{prop:comp-capr}
Given program $P$ satisfying $\Delta^r$ for $A,B \subseteq \Lits$, 
	(i) given	
	SE-interpretation $\htm{X, Y}$ with $Y \subseteq \Lits \setminus A$, deciding whether $X \in \bigcap \mathcal{R}^Y_{\htm{P,A,B}}$ is in $\Pi^P_2$;
	(ii) deciding whether $P$ satisfies criterion $\Omega_{A,B}$ is $\Sigma^P_3$-complete.
\end{prop}
\begin{proof}[Proof (Sketch)]
For the complemetary problem, 
it suffices to guess an $A' \supseteq A\cap B$, and check that 
$\htm{Y\cup A',Y\cup A'}\in\SE^B(P)$ and
$\htm{X \cup (A\cap B),Y\cup A'}\notin\SE^B(P)$ (the former ensures also that $\mathcal{R}^Y_{\htm{P,A,B}}\neq \emptyset$).
(i) then follows by $B$-SE-model checking being in $D^P$. %\cite{eiter2007semantical}.
For (ii), we just need to additionally guess $Y$ 
and check that $\mathcal{R}^Y_{\htm{P,A,B}}$ is non-empty (see above) and has no least element. For the latter, we additionally guess $X$ and use (i) together with Proposition~\ref{prop:omega-a-b-se}.
$\Sigma^P_3$-hardness follows from the special case $B=\overline{A}$, cf.\ Thm 16, \cite{DBLP:journals/ai/GoncalvesKLW20}, where $\Delta^r$ is trivially satisfied.
\end{proof}

\nop{
SW: remove next two results below, if above ok.

\begin{prop}\label{prop:comp-capr:old}
Given program $P$ satisfying $\Delta^r$ for $A,B \subseteq \Lits$, and SE-interpretation $\htm{X, Y}$ with $Y \subseteq \Lits \setminus A$, deciding whether $X \in \bigcap \mathcal{R}^Y_{\htm{P,A,B}}$ is in $D^P_2$ .
\end{prop}

Basically, we have to perform a $\Sigma^P_2$-test that decides whether $\mathcal{R}^Y_{\htm{P,A,B}}\neq \emptyset$ and a $\Pi^P_2$-test that checks
for all $A'$ with $A\cap B \subseteq A'\subseteq  A$ , either $\htm{Y \cup A' , Y \cup A'} \notin \SE^B(P)$ or $\htm{X \cup (A \cap B) , Y \cup A'} \in \SE^B(P)$ .

%We finally provide the complexity for criterion $\Omega_{A,B}$.

\begin{thm}\label{thm:comp-cri}
Given program $P$ satisfying $\Delta^r$ for $A,B \subseteq \Lits$. Deciding whether $P$ satisﬁes criterion $\Omega_{A,B}$ is in $\Sigma^P_3$.
\end{thm}

For this we guess and interpretation $Y$ and check that $\mathcal{R}^Y_{\htm{P,A,B}}$ is non-empty and has no least element.

SW: remove ends
}

Recall Corollary~\ref{cor:main}. The results in Proposition \ref{prop:compdelta} and Proposition \ref{prop:comp-capr} are then used to determine the complexity of deciding relativized simplifiability.

\begin{thm}
	Let $P$ be a program over $\Lits$, and $A,B \subseteq \Lits$. Deciding whether $P$ is $A$-simplifiable relative to $B$ is in $\Pi^P_{3}$.
\end{thm}
%Recall Corollary~\ref{cor:main}. Deciding whether 
%$P$ satisfies $\Delta^r$ is in $\Pi^P_2$ (Proposition \ref{prop:compdelta}), deciding whether
%$P$ violates criterion $\Omega_{A,B}$ is in $\Pi^P_3$ (Proposition \ref{prop:comp-capr}).
%$\Pi^P_{3}$-membership thus follows.

By making use of the characterizing equality $\eqref{eq:se_rel_proj}$,  Propositions~\ref{prop:omega-a-b-se} and \ref{prop:comp-capr}, we %next 
finally provide the complexity result for relativized simplification testing.

\begin{thm}
Given program $P$ which is $B$-relativized $A$-simplifiable, and program $Q$, checking whether $Q$ is a $B$-relativized $A$-simplification of $P$ is $\Pi^P_3$-complete. 
\end{thm}

\begin{proof}[Proof (Sketch)]
Making use of the %characterizing 
equality 
$\eqref{eq:se_rel_proj}$, we show the complementary problem to be in $\Sigma^P_3$, by guessing an SE-interpretation $\htm{X,Y}$ and checking the containment in $\SE^B_A(P)$ or in $\SE^{\overline{A},B}(Q)$ but not both. By Proposition~\ref{prop:omega-a-b-se}, deciding $\htm{X,Y} \in SE^B_A(P)$ amounts to checking that $Y\subseteq \Lits\setminus A$ and $X\in\bigcap \mathcal{R}^Y_{\htm{P,A,B}}$. By Proposition~\ref{prop:comp-capr} the latter is in $\Pi^P_2$, 
	%$D^P_2$, 
	and $B$-SE-model checking is in $D^P$.

For hardness, we use the case $B=\emptyset$, where the problem reduces to decide
	$\AS(P)|_{\overline{A}}=\AS(Q)$ for a program $P$ being $\emptyset$-relativized $A$-simplifiable. 
	To this end, we extend the hardness-construction from \cite{DBLP:journals/amai/EiterG95} and 
	reduce from $(3,\forall)$-QSAT. Let $\Phi=\forall U \exists V \forall W \phi$ with $\phi=\bigvee_{i=1}^n(l_{i,1} \wedge l_{i,2} \wedge l_{i_3})$. 
	%W.l.o.g. we assume that $U,V,W$ are all non-empty and that
	%each clause
	%$l_{i,1} \wedge l_{i,2} \wedge l_{i_3}$
        %contains at least one literal over $V\cup W$.
		We use copies of atoms, e.g., $\widetilde{U}=\{\widetilde{u}\mid u \in U\}$.
	We construct $P$ as follows:
\begin{eqnarray*}
	P \!\!\!\!&=\!\!\!\!&  
	\{
		x\vee \widetilde{x} \la.\ \mid x\in U\cup V \cup W \} \cup\\
	&&\{	w \la s.\ \widetilde{w} \la s.\ s\la w,\widetilde{w}.\ \mid w\in W \} \cup\\
	&&\{ 	s \la \widetilde{l}_{i,1}, \widetilde{l}_{i,2}, \widetilde{l}_{i,3}.\ \mid 1\leq i \leq n\}\cup \\
	&&\{	\bot \la \mathit{not}\ s.\ \} 
\end{eqnarray*}
	where $\widetilde{l}_{i,j}$ is given by $\widetilde{a}$ if $l_{i,j}$ is $\neg a$ and
                                by $\widetilde{l_{i,j}}=l_{i,j}$ if $l_{i,j}$ is a positive literal.
				%\todo{Is program $P$ $\emptyset$-relativized $A$-simplifiable no matter how $\Phi$ looks like? I guess yes: YES}
				Note that $P$ is $\emptyset$-relativized $A$-simplifiable no matter how $\Phi$ looks like.
	Moreover, we set 
	$$
		Q = \ \{ u\vee \widetilde{u} \la.\ \mid u\in U\}
	$$
	and $A=V\cup\widetilde{V}\cup W\cup\widetilde{W}\cup \{s\}$.
	It can be checked that 
	$\AS(P)|_{\overline{A}}=\AS(Q)$ holds iff $\Phi$ is true.
\end{proof}

\begin{table}
%{\small
%\begin{center}
\begin{tabular}{c|ccc}
 & $\emptyset$ & $B$ & $\Lits$\\
\hline 
$\emptyset$ &$\Pi^P_2$-complete &$\Pi^P_2$-complete  &$\coNP$-complete \\[1pt]
$A$&$\boldsymbol{\Pi^P_3}$\textbf{-complete} &$\boldsymbol{\Pi^P_3}$\textbf{-complete}& in $\Pi^P_2$\\
\hline
%$\Lits$&&&
\end{tabular}
%\end{center}
%}
\caption{Complexity landscape of testing $B$-relativized $A$-simplification. Results in bold-face are given in this paper.} %First row refers to equivalence, relativized strong equivalence and strong equivalence, respectively. Second row refers to faithful abstractions, relativized strong simplifications and strong simplifications, respectively.}
\label{table:comp}
\end{table}

Table~\ref{table:comp} presents the full complexity landscape, using results from  \cite{DBLP:journals/amai/EiterG95,eiter2007semantical} for (relativized) (strong) equivalence and \cite{zgsswkr23} for strong simplifications. \citeauthor{zgskr18} \shortcite{zgskr18} provides the $\Pi^P_2$-completeness of $B=\emptyset$ for normal programs, which here we lifted to ELPs.

\section{Conclusion}

We introduced a novel relativized equivalence notion, which is a relaxation of the recent strong simplification notion \cite{zgsswkr23}, that provides a unified view over all related notions of forgetting and strong equivalence in the literature. We provided the necessary and sufficient conditions to ensure such relativized simplifiability. We observed that the challenge is when the context programs do not contain all the atoms to remove, that requires a criterion on the program that focuses on its relativized SE-models, which also captures the case of \SP-forgetting.

We furthermore introduced an operator that can obtain relativized simplifications, when possible. We showed that at least for those atoms to be removed that appear in the context programs, it is possible to simply project them away, while for those that are outside of the context programs a relaxed version of an \SP-forgetting operator will need to be applied. We provided complexity results that fill the gap in the landscape of the introduced notion.

Investigating the relativized simplification notion for the uniform case to bring together uniform persistence forgetting \cite{gonccalves2019forgetting}, (relativized) uniform equivalence and uniform simplifications \cite{zgsswkr23} would be an interesting extension of this work.

\section*{Acknowledgments}

This work has been supported by the Austrian Science Fund (FWF) projects T-1315, P32830,
and by the Vienna Science and Technology Fund (WWTF) under grant ICT19-065.
We thank the anonymous reviewers for their valuable feedback.

\bibliographystyle{aaai24}
\bibliography{ref}

\appendix
\section{Proofs}

Below observation is similar to Proposition~2 from \cite{zgsswkr23} but adjusted to the relativized case.

\begin{prop}\label{prop:s-rel-condition}
If there is a strong $A$-simplification of $P$ relative to $B$, %If $P$ is uniform $A$-simplifiable, 
then $P$ satisfies the condition
\begin{align*}
\forall Y \hbox{(over $\Lits$)} \forall Z \hbox{(over $B$)}, Y \in \AS(P\cup Z),\nonumber\\ \forall Z' \hbox{(over $B$)}, Z'_{|\overline{A}}=Z_{|\overline{A}}\label{eq:s-red-desiredcase}\\
\exists Y' \hbox{(over $\Lits$)}, Y'_{|\overline{A}}=Y_{|\overline{A}}, Y' \in \AS(P\cup Z')\nonumber %nonumber
\end{align*}
\end{prop}

\begin{proof}[Proof of Prop.~\ref{prop:se-rel-deltas}]
Let $Q$ be a strong $A$-simplification of $P$ relative to $B$, but some
%
%$\Delta_{s_1}$: Let $\htm{Y,Y}$ violate this. 
$\htm{Y,Y}$ violates $\Delta_{s_1}^r$. Then $\htm{Y\cup (A\cap B),Y\cup (A\cap B)} \in \HT^B(P)$ needs to hold. Since otherwise we would have $AS(P\cup ((Y\cup A) \cap B))=\emptyset$ while $Y \in AS(P\cup (Y\cap B))$. 
%Then by Proposition~\ref{prop:1}, $\htm{Y,Y\cup (A\cap B)} \in \HT(P)$ holds. Now consider $R=((Y\cup A) \cap B))\cup\{\bot \leftarrow y \mid y \notin (Y\cup A) \cap B\}$ and $R'=(Y \cap B) \cup \{\bot \leftarrow y \mid y \notin (Y\cup A) \cap B\} \cup \{y \leftarrow \mi{not}\ y \mid y \in ((Y\cup A) \cap B) \setminus (Y \cap B)\}$. We have $\htm{Y \cup (A\cap B),Y\cup (A\cap B)} \in AS(P\cup R)$, thus $\htm{Y_{|\overline{A}},Y_{|\overline{A}}} \in AS(Q \cup R_{|\overline{A}})$. Both $R$ and $R'$ are $A$-separated with $R_{|\overline{A}} = R'_{|\overline{A}}$, but $AS(P\cup R')=\emptyset$, contradicting Proposition~\ref{prop:s-rel-condition}.
%
Now consider $R=((Y\cup A) \cap B))\cup\{f \leftarrow y, \mi{not}\ f \mid y \notin (Y\cup A) \cap B\} \cup \{\epsilon \leftarrow a, \mi{not}\ \epsilon \mid a \in A \cap B \}$ and $R'=(Y \cap B) \cup \{f \leftarrow y, \mi{not}\ f \mid y \notin (Y\cup A) \cap B\}$, where $f$ is an auxiliary atom in $\overline{A}$, and $\epsilon$ is an auxiliary atom to be removed together with $A$.
We have $Y \in AS(P \cup R')$ while $AS(P\cup R)=\emptyset$, contradicting Proposition~\ref{prop:s-rel-condition}.

$\Delta_{s_2}^r$:
%Let $\htm{X,Y}{\in}\SE^B(P)$ violate this. So there is $\htm{X',Y}\in \SE(P)$ with $X'\cap B=X$ and 
Let $\htm{X',Y}{\in}\SE(P)$ violate this. So $\htm{Y,Y}\in \SE^B(P)$ and
$X'_{|\overline{A}}\,{=}\,Y_{|\overline{A}}$  but $X' \neq Y$. So $X'\subset Y$ needs to hold where $Y\setminus X' \subseteq A$. Consider $R=(Y\cap B)\cup\{f \leftarrow y, \mi{not}\ f \mid y \notin Y\cap B \wedge y \notin A\cap B\} \cup \{\epsilon \leftarrow y, \mi{not}\ \epsilon \mid y \notin Y\cap B \wedge y \in A\cap B\}$, where $f$ is an auxiliary atom in $\overline{A}$, and $\epsilon$ is an auxiliary atom which will also be omitted together with $A$. We have $\htm{Y,Y} \in AS(P\cup R)$, thus $\htm{Y_{|\overline{A}},Y_{|\overline{A}}} \in AS(Q \cup R_{|\overline{A}})$. Let $X= X'\cap B$. Now consider $R'=X \cup\{f \leftarrow y, \mi{not}\ f \mid y \notin Y\cap B \wedge y \notin A\cap B\} \cup \{\epsilon \leftarrow y, \mi{not}\ \epsilon \mid y \notin Y\cap B \wedge y \in A\cap B\} \cup \{y \leftarrow \mi{not}\ y \mid y \in (Y\cap B)\setminus X\}$. Both $R$ and $R'$ are $A$-separated, while $R_{|\overline{A}} = R'_{|\overline{A}}$, and
$AS(P\cup R')=\emptyset$, contradicting Proposition~\ref{prop:s-rel-condition}.

$\Delta_{s_3}^r$:  Assume for some $\htm{X,Y}\in\HT^B(P)$, $\htm{X \cup (Y \cap (A\cap B)),Y} \notin \HT^B(P)$.
So there is no $\htm{X',Y}\in \HT(P)$ with $X' \cap B = X \cup (Y \cap (A\cap B))$, while there is some $\htm{X'',Y}\in\HT(P)$ with $X''\cap B=X$.
 Now consider $R= X \cup \{f\leftarrow y, \mi{not}\ f \mid y \notin Y\cap B \wedge y \notin A\cap B\} \cup \{\epsilon \leftarrow y, \mi{not}\ \epsilon \mid y \notin Y\cap B \wedge y \in A\cap B\} \cup \{y_1 \leftarrow y_2 \mid y_1,y_2 \in ((Y\cap B) \setminus X)\setminus A \}$
and $R'= (X \cup (Y \cap (A\cap B))) \cup \{f \leftarrow y, \mi{not}\ f \mid y \notin Y \wedge y \notin A\} \cup \{\epsilon \leftarrow y, \mi{not}\ \epsilon \mid y \notin Y \wedge y \in A\} \cup \{y_1 \leftarrow y_2 \mid y_1,y_2 \in ((Y\cap B) \setminus X)\setminus A  \}$, where $f$ is an auxiliary atom in $\overline{A}$, and $\epsilon$ is an auxiliary atom which will also be omitted together with $A$.
We have $AS(P\cup R)=\emptyset$ and $AS(P\cup R')=\{Y\}$ while $R_{|\overline{A}} = R'_{|\overline{A}}$, contradicting Proposition~\ref{prop:s-rel-condition}.
\end{proof}

\begin{proof}[Proof of Prop.~\ref{prop:se-q-rel-1}]
By Proposition~\ref{prop:se-rel-deltas}, $P$ satisfies $\Delta_{s_1}^r,\Delta_{s_2}^r$ and $\Delta_{s_3}^r$.
%(Only-if direction) 

($SE^B_A(P)\subseteq SE^{\overline{A},B}(Q)$). Suppose that for some $\htm{X,Y} \in SE^B(P)$, $\htm{X,Y}_{|\overline{A}} \in SE^B_A(P)$, %\new{with $X \cap A = Y \cap A$} 
but $\htm{X,Y}_{|\overline{A}} \notin SE^{\overline{A},B}(Q)$. %Now consider $\htm{Y,Y}$.
Case $X=Y$: %$\htm{Y,Y}_{|\overline{A}}\notin SE^{\overline{A},B}(Q)$. 
Either (1) $Y_{|\overline{A}}\nmodels Q$ or (2) there exists $Y' \subset Y_{|\overline{A}}$ with $Y'\cap B = Y_{|\overline{A}}\cap  B$ such that $Y' \models Q^{Y_{|\overline{A}}}$. We know that $\htm{Y,Y}\in \HT(P)$. We take $R=Y\cap B$, which is $A$-separated since it is a set of facts. 
Then $\htm{Y,Y} \models P\cup R$, and for no subset $Z \subset Y$, $\htm{Z,Y} \models P\cup R$; $Y$ is an answer set of $P\cup R$. If case (1) then $\htm{Y,Y}_{|\overline{A}}\notin \HT^{\overline{A}}(Q \cup R_{|\overline{A}})$, if case (2) then $\htm{Y',Y_{|\overline{A}}}\in \HT^{\overline{A}}(Q \cup R_{|\overline{A}})$. Thus $Y_{|\overline{A}}$ is not an answer set of $Q\cup R_{|\overline{A}}$, contradicting the assumption on $Q$ being a strong $A$-simplification of $P$ relative to $B$. 

Case $X\subset Y$: We have $\htm{Y,Y}_{|\overline{A}}\in SE^{\overline{A},B}(Q)$. Consider $\htm{X',Y}\in \SE(P)$ with $X'\cap B=X$. Since $P$ satisfies $\Delta_{s_2}^r$ we know that if $X'\subset Y$ then $X'_{|\overline{A}}\subset Y_{|\overline{A}}$ holds. We define $R=X \cup \{C_2 \leftarrow C_1 : C_1,C_2 \in ((Y\setminus A) \cap B) \setminus X, C_1\neq C_2\}$.
Consider $R_{|\overline{A}}$, since $R$ does not contain atoms from $A$,  %(i.e., $Y\setminus X \cap A = \emptyset$),
 it is $A$-separated
and since $\htm{Y,Y} \in \HT(R)$, %then by Lemma~\ref{prop:all}(a) 
$\htm{Y,Y}_{|\overline{A}} \in \HT^{\overline{A}}(R_{|\overline{A}})$, thus it is an SE-model of $Q\cup R_{|\overline{A}}$. 

Consider an SE-interpretation $\htm{Z',Y_{|\overline{A}}}$ with $Z' \subset Y_{|\overline{A}}$. We will check if it can be an SE-model of $Q\cup R_{|\overline{A}}$. If $Z'\cap B \subset X'_{|\overline{A}} \cap B$, then $Z' \nmodels R^{Y_{|\overline{A}}}_{|\overline{A}}$. For the case $Z'\cap B \supseteq X'_{|\overline{A}} \cap B$, $X'_{|\overline{A}} \cap B = Z' \cap B$ cannot hold, since by assumption $\htm{X'\cap B,Y}_{|\overline{A}}=\htm{X,Y}_{|\overline{A}}\notin SE^{\overline{A},B}(Q)$. So $Z'\cap B \supset X'_{|\overline{A}} \cap B$. However, again $Z' \nmodels R^{Y_{|\overline{A}}}_{|\overline{A}}$, since some $C_1 \in (Z'\setminus X')\cap B$ and $C_2\in(Y_{|\overline{A}}\setminus Z')\cap B$ can be found with $C_2\leftarrow C_1 \in R_{|\overline{A}}$. So $\htm{Y_{|\overline{A}},Y_{|\overline{A}}}$ must be an answer set of $Q \cup R_{|\overline{A}}$.

%Now consider for any $A' \subseteq A$, the atoms %for $A' = Y \cap A$ the atoms 
%$Z' \cup A'$, for which we have $Z' \cup A' \subset Y$, and $X \subset Z' \cup A'$ since $X_{|\overline{A}} \subset Z'$. Then we pick some $C \in Z' \setminus X$ and $B \in (Y\setminus A) \setminus (Z'\cup A')$. 
%For these atoms $B \leftarrow C$ belongs to $R$, but $\htm{Z'\cup A',Y}$ does not satisfy this implication, thus since \textcolor{blue}{by Lemma~\ref{prop:all}(b) as $\htm{Z'\cup A',Y} \notin \HT(R)$ for any $A' \subseteq A$}, $\htm{Z',Y_{|\overline{A}}} \notin \HT^{\overline{A}}(R_{|\overline{A}})$, contrary to the assumption that it is an SE-model of  $Q\cup R_{|\overline{A}}$. So $\htm{Y_{|\overline{A}},Y_{|\overline{A}}}$ must be an answer set of $Q \cup R_{|\overline{A}}$.
%
%Now consider the SE-model $\htm{X,Y}$ of $P$. Clearly it 
Now consider $\htm{X',Y}\in\SE(P)$. Clearly it satisfies each implication $C_2 \leftarrow C_1 \in R$. So $\htm{X',Y} \models P\cup R$, %On the other hand $X \neq Y$ because $\htm{Y,Y}_{|\overline{A}}$ is an SE-model of $Q$ but $\htm{X,Y}_{|\overline{A}}$ is not. So $X\subset Y$ 
thus $Y$ is not an answer set of $P\cup R$. Also since $\htm{X,Y}_{|\overline{A}}\in SE^B_A(P)$ we know that for any $\htm{Y',Y'}\in SE^B(P)$ with $Y'_{|\overline{A}}=Y_{|\overline{A}}$, we have $\htm{X_1,Y'} \in SE^B(P)$ with $X_{|\overline{A}}=X_{1|\overline{A}}$. Also since $P$ satisfies $\Delta^r_{s_3}$, $\htm{X_1 \cup (Y' \cap (A \cap B)),Y'}\in \SE^B(P) $. So any such $\htm{X'',Y'} \in SE(P)$ with $X'' \cap B=X_1 \cup (Y' \cap (A \cap B))$ is an SE-model of $P\cup R$. Thus any potential $Y'$ is not an answer set of $P\cup R$, contradicting the assumption. % on $Q$ being a strong $A$-simplification of $P$ relative to $B$.

($\SE^{\overline{A},B}(Q) \subseteq \SE^B_A(P)$). Suppose $\htm{X,Y}\in \SE^{\overline{A},B}(Q)$ but $\htm{X,Y}\notin \SE^B_A(P)$. We take into account that $P$ satisfies $\Delta_{s_1}^r$. So either (i) for all $A'\supseteq A\cap B$, $\htm{Y\cup A',Y\cup A'} \notin \HT^B(P)$, or (ii) for some $A' \supseteq A\cap B$, $\htm{Y\cup A',Y \cup A'} \in \HT^B(P)$ but %\rev{$\htm{X'\cup A'',Y'\cup A'}$ is not an SE-model of $P$ for any $X' \cup A''$}
for some $A''$, $\htm{X \cup (A''\cap B), Y \cup A''} \notin \HT^B(P)$ while $\htm{Y\cup A'',Y\cup A''}\in \HT^B(P)$ (Since $P$ satisfies $\Delta_{s_3}^r$, this means $\htm{X \cup A''', Y\cup A''} \notin \HT^B(P)$, for all $A'''\subseteq A''$, and thus $\htm{X,Y}$ does not appear in $\SE^B_A(P)$).

(Case i) We know that $\htm{Y,Y}\in \SE^{\overline{A}}(Q)$. Let %a set 
$R'=Y\cap B$. %of facts. 
$\htm{Y,Y} \in \SE^{\overline{A}}(Q\cup R')$, and for no subset $Z \subset Y$, $\htm{Z,Y}$ is an SE-model of $Q\cup R'$; $Y$ is an answer set of $Q\cup R'$. Now for any $R$ over $B$ that is $A$-separated which can be mapped to $R'$, for any $A' \supseteq A\cap B$, $\htm{Y\cup A',Y\cup A'}$ % \nmodels P \cup R$, thus 
is not an answer set of  $P \cup R$, contradicting the assumption. % on $Q$ being a strong $A$-simplification of $P$ relative to $B$.

(Case ii) Let $\htm{Y\cup A'',Y\cup A''}\in\HT^B(P)$, thus also in $\HT(P)$. Consider $R' =X \cup \{C_2 \leftarrow C_1 \mid C_1,C_2 \in (Y\cap B) \setminus X, C_1\neq C_2\}$. We have $\htm{Y,Y} \in \HT^{\overline{A}}(R')$.
Let $R=(X \cup (A''\cap B)) \cup \{C_2$ %\cup A' 
$\leftarrow C_1 \mid C_1,C_2 \in (Y \cup A'')\cap B \setminus X \cup (A''\cap B)), C_1\neq C_2\}$ 
%for some $A'$ such that $A' \subseteq A$, 
be an $A$-separated %\textcolor{red}{unary!?!}
program which can be mapped to $R'$. So $\htm{Y\cup A'',Y \cup A''} \models P \cup R$. 

Consider an SE-interpretation $\htm{Z,Y \cup A''}$ with $Z \subset Y\cup A''$. We will check if it can be an SE-model of $P\cup R$. If $Z\cap B \subset X \cup (A''\cap B)$, then $Z \nmodels R^{Y\cup A''}$. If $Z\cap B \supseteq X \cup (A''\cap B)$, clearly, %\rev{for any $A''$, $X' \cup A''$}
$X \cup (A''\cap B)= Z \cap B$ cannot hold since by assumption $\htm{X \cup (A''\cap B),Y\cup A''}$ is not in $\SE^B(P)$. So $Z\cap B \supset X \cup (A''\cap B)$. However again $Z \nmodels R^{Y\cup A''}$, since some $C_1 \in (Z\setminus X \cup (A'' \cap B))\cap B$ and $C_2 \in (Y\cup A''\setminus Z)\cap B$ can be found with $C_2\leftarrow C_1 \in R$. So $\htm{Y \cup A'', Y \cup A''}$ must be an answer set of $P \cup R$.

Now consider the SE-model $\htm{X',Y}$ of $Q$ with $X'\cap B=X$. Clearly %it is an SE-model of $X$ and satisfies each implication $C_2 \leftarrow C_1 \in R'$. So 
$\htm{X',Y}\in \SE^{\overline{A}}(Q\cup R')$. On the other hand, $X'\neq Y$, because $\htm{Y \cup A'', Y\cup A''}$ is in $SE^B(P)$ but %\rev{none of $\htm{X'\cup A'',Y'\cup A'}$ is}
$\htm{X\cup (A''\cap B), Y \cup A''}$ is not. So $X' \subset Y$, thus $\htm{Y,Y}$ is not an answer set of $Q \cup R'$, contradicting the assumption. % on $Q$ being a strong $A$-simplification of $P$ relative to $B$.
\end{proof}

\begin{proof}[Proof of Prop.~\ref{prop:criterion}]
    There exists some $\htm{Y,Y} \in \SE^B_A(P)$ such that for all $A'$ with $\htm{Y\cup A',Y\cup A'} \in \SE^B(P)$, $\{\htm{X,Y} \mid \htm{X\cup A'',Y\cup A'} \in SE^B(P)\}\neq \{\htm{X,Y} \mid \htm{X,Y} \in SE^B_A(P)\}$. This means for any such $\htm{Y\cup A',Y\cup A'} \in \SE^B(P)$ there exists some $\htm{X_{A'},Y\cup A'}\in \SE^B(P)$ where $\htm{X_{A'|\overline{A}},Y} \notin \SE^B_A(P)$. 
    
    Say such a $Q$ exists. By Proposition~\ref{prop:se-q-rel-1}, we know that $SE^B_A(P)=SE^{\overline{A},B}(Q)$ should hold. So $\htm{X_{A'|\overline{A}},Y} \notin SE^{\overline{A},B}(Q)$ while $\htm{Y,Y}\in SE^{\overline{A},B}(Q)$.
    Now consider a program $R$ over $B$ whose SE-models that only contain atoms from $\overline{A} \cap B$ are exactly $\htm{Y\cap B,Y\cap B}$  for $\htm{Y\cup A',Y\cup A'} \in SE^B(P)$ and $\htm{X_{A'|\overline{A}},Y \cap B}$  for $\htm{X_{A'},Y\cup A'}\in \SE^B(P)$. Such an $A$-separated program can be constructed which does not contain any atoms from $A$. Then none of $Y\cup A'$ will be an answer set of $P\cup R$, while $\htm{Y,Y} \in AS(Q\cup R_{|\overline{A}})$, which contradicts \eqref{eq:s-rel}.
    %\todo[inline]{due to $A$-separation requirement $R$ might not "exactly" have those SE-models when $B\cap A\neq \emptyset$, but have also some that agree on the projection.}
\end{proof}

Below Lemma is similar to Lemma 27 from \cite{zgsswkr23} (supplementary file) where in this paper we have $A$-separated programs over $B$, instead of $\Lits$.

\begin{lemma}\label{prop:all}
An $A$-separated program $R$ over $B$ satisfies the following properties.
\be[(a)]
\item If $X \models R^Y$ then $X_{|\overline{A}} \models R_{|\overline{A}}^{Y_{|\overline{A}}}$.
\item If $\htm{X,Y} \in \HT^{\overline{A}}(R_{|\overline{A}})$ then for any $A' \subseteq A\cap B$ where $\htm{Y\cup A',Y\cup A'} \in \HT(R)$, there exists some $A'' \subseteq A'$ such that $\htm{X\cup A'',Y\cup A'} \in \HT(R)$.
\item $\htm{X,Y}\in \HT(R) \Rightarrow \htm{X\cup (Y\cap A),Y}\in \HT(R)$.
\item $\htm{Y,Y}\in \HT(R) \Rightarrow \htm{Y\cup (A\cap B),Y \cup (A\cap B)}\in \HT(R)$.
\item If $\htm{Y,Y} \in \HT^{\overline{A}}(R_{|\overline{A}})$ then for some $A'\subseteq (A\cap B)$, $\htm{Y\cup A',Y\cup A'} \in \HT(R)$.
\ee
\end{lemma}

\begin{proof}[Proof of Thm.~\ref{thm:se-q-rel}]
We show that a program $Q$ satisfying 
%\beq
%SE^B(P)_{|\overline{A}}=SE^{\overline{A},B}(Q)
%\eeq{eq:se_rel_proj} 
\eqref{eq:se_rel_proj}
is an $A$-simplification of $P$ relative to $B$.

(Case i) Suppose $Z \in AS(P \cup R)_{|\overline{A}} $, but $Z \notin AS(Q \cup R_{|\overline{A}})$. By $\Delta_{s_1}$, $\htm{Z\cup A',Z\cup A'}\in \HT^B(P)$ for $A' \supseteq A \cap B$. % (and there is no $\htm{Z\cup A'',Z\cup A''}\in \HT^B(P)$ with $A''\subset A \cap B$). 
Thus $\htm{Z,Z}\in \HT^B_A(P)$. We either have (i-1) $Z  \nmodels Q \cup R_{|\overline{A}}$, or (i-2) there exists $Z' \subset Z$ s.t. $Z' \models (Q \cup R_{|\overline{A}})^{Z}$.
(i-1) Assume $Z  \nmodels Q \cup R_{|\overline{A}}$. By Lemma~\ref{prop:all}(a), since $Z\cup A' \models R$, we get $Z \models R_{|\overline{A}}$. So $Z  \nmodels Q$ holds, %which means $\htm{Z \cup A,Z \cup A} \in \HT(P)$ while $\htm{Z,Z} \notin \HT^{\overline{A}}(Q)$ 
which contradicts \eqref{eq:se_rel_proj}.
(i-2) Assume there exists $Z' \subset Z$ s.t. $Z' \models (Q \cup R_{|\overline{A}})^{Z}$. So we have $\htm{Z'\cap B,Z} \in \HT^{\overline{A},B}(Q)$, and $Z' \models R_{|\overline{A}}^{Z}$. %By $m$-robustness of $R$ 
By Lemma~\ref{prop:all}(b), we know that for some $A''\subseteq A' \subseteq A \cap B$, $Z' \cup A'' \models R^{Z\cup A'}$, %
and by Lemma~\ref{prop:all}(c) we get $Z' \cup A' \models R^{Z\cup A'}$. Actually since $R$ is over $B$, for any $Z''$ with $Z''\cap B = (Z'\cup A')\cap B$ it holds that $Z''\models R^{Z\cup A'}$. Now since $Z \cup A \in AS(P \cup R)$, we know that for every $T \subset Z \cup A$, $T \nmodels (P\cup R)^{Z\cup A}$. So for any $Z''$ with $Z''\cap B = (Z'\cup A)\cap B$,  $Z'' \nmodels P^{Z \cup A}$, thus $\htm{(Z'\cup A)\cap B,Z\cup A} \notin \HT^B(P)$. 
By condition $\Delta_{s_3}$ this means that $\htm{(Z' \cup A')\cap B,Z\cup A}\notin \HT^B(P)$ for any $A' \subseteq A$, since otherwise $\Delta_{s_3}$ would force to have $\htm{(Z'\cup A)\cap B,Z\cup A} \in \HT^B(P)$. This means $\htm{Z'\cap B,Z} \notin \HT^B_A(P)$ which contradicts \eqref{eq:se_rel_proj}.

(Case ii) Suppose $Z \in AS(Q \cup R_{|\overline{A}})$, but $Z \cup A' \notin AS(P \cup R)$ for any $A' \subseteq A$. So we have that $\htm{Z,Z} \in \HT^{\overline{A},B}(Q)$, and $\htm{Z,Z} \in \HT^B_A(P)$.
Since $P$ does not satisfy $\Omega_{A,B}$, there is some $A' \in \mi{Rel}^Z_{P,A,B}$. So we take the relevant $\htm{Z\cup A',Z\cup A'}\in \HT^B(P)$ for $\htm{Z,Z} \in \HT^B_A(P)$. By $\Delta_{s_1}$, we know that $A' \supseteq A\cap B$. % (and there is no $\htm{Z\cup A'',Z\cup A''}\in \HT^B(P)$ with $A''\subset A \cap B$). 
We know that $Z \models R_{|\overline{A}}$, and by Lemma~\ref{prop:all}(b)-(d) we know that $Z \cup (A\cap B) \models R$. So we get $Z \cup A' \models P\cup R$ for any $A' \supseteq A \cap B$.

So to have $Z \cup A'' \notin AS(P \cup R)$ for any $A'' \subseteq A$, we look at $Z \cup A' \models P\cup R$. It should hold that there exists $Z' \subset Z \cup A'$ s.t. $Z' \models (P \cup R)^{Z \cup A'}$.
%Assume there exists $Z' \subset Z \cup A'$ s.t. $Z' \models (P \cup R)^{Z \cup A'}$, for any $A'$ with $Z \cup A'  \models P \cup R$. %By $m$-robustness of $R$ we know that $Z'_{|\overline{A}} \models R_{|\overline{A}}^Z$.
So we have $\htm{Z', Z\cup A'} \in \HT(P)$ (thus $\htm{Z'\cap B, Z\cup A'} \in \HT^B(P)$) and $Z' \models R^{(Z \cup A)}$.
Since $A' \in \mi{Rel}^Z_{P,A,B}$, %$\htm{Z\cup A',Z\cup A'}\in \HT^B(P)$ satisfies $\Omega_{A,B}$ for $\htm{Z,Z} \in \HT^B_A(P)$, then 
we know that $\htm{(Z'\cap B)_{|\overline{A}},Z} \in \SE^B_A(P)$. 
Since $P$ satisfies $\Delta_{s_2}$, we know that $Z'_{|\overline{A}} = Z$ cannot hold, so we look at the case $Z'_{|\overline{A}} \subset Z$. By Lemma~\ref{prop:all}(a) we know that $Z'_{|\overline{A}} \models R_{|\overline{A}}^Z$. Actually for any $Z'' \subseteq Z$ with $Z'' \cap B_{|\overline{A}} = Z'_{|\overline{A}} \cap B_{|\overline{A}}$, $Z'' \models R_{|\overline{A}}^Z$. Since $Z \in AS(Q \cup R_{|\overline{A}})$, we know that for every $T \subset Z$, $T \nmodels (Q \cup R_{|\overline{A}})^Z$. 
So also for any $Z'' \subseteq Z$ with $Z'' \cap B_{|\overline{A}} = Z'_{|\overline{A}} \cap B_{|\overline{A}}$, $Z'' \nmodels Q^Z$, making $\htm{(Z' \cap B)_{|\overline{A}},Z} \notin SE^{\overline{A},B}(Q)$ which contradicts \eqref{eq:se_rel_proj}. 

For the reason on why it was enough to only look at $Z\cup A'$, note that $\htm{(Z'\cap B)_{|\overline{A}},Z} \in \SE^B_A(P)$ means for any $\htm{Z\cup A'',Z\cup A''} \in \SE^B(P)$ there is some $\htm{X,Z\cup A''}\in \SE^B(P)$ where $X_{|\overline{A}}=(Z'\cap B)_{|\overline{A}}$, so no $Z\cup A''$ can be an answer set of $P\cup R$.
\end{proof}

\begin{proof}[Proof of Prop.~\ref{prop:b-se-proj}]
We begin with $\SE^{\overline{A\cap B},B\setminus A}(P_{|\overline{A \cap B}}) = \SE^B(P)_{|\overline{A\cap B}}$. The direction $\subseteq$ is easy to see.

Thus we show $\SE^B(P)_{|\overline{A\cap B}}\subseteq \SE^{\overline{A\cap B},B\setminus A}(P_{|\overline{A \cap B}})$.
Assume for $\htm{X,Y} {\in} SE^B(P)_{|\overline{A\cap B}}$ 
it does not hold.
As $P$ satisfies $\Delta^r_{s_1}$ and $\Delta^r_{s_3}$ we know that there is $\htm{X \cup (A\cap B),Y\cup A'} \in SE^B(P)$, for some $A' \supseteq A\cap B$. 

Either (1) $\htm{X',Y} \notin \SE^{\overline{A\cap B}}(P_{|\overline{A\cap B}})$, for any $X'$ with $X'\cap B = X$, or (2) $\htm{Y,Y} \in \SE^{\overline{A\cap B}}(P_{|\overline{A\cap B}})$ but $\htm{Y,Y} \notin \SE^{\overline{A\cap B},B\setminus A}(P_{|\overline{A \cap B}})$.

Case (1): There is a rule $r {\in} P_{|\overline{A\cap B}}$ such that $X' {\models} B(r^Y)$ but $X' {\nmodels} H(r^Y)$. Now consider the rule $r' {\in} P$ such that $r'_{|\overline{A \cap B}}{=}r$. We know that $B^-(r'){\cap} (A\cap B) = \emptyset$ and $H(r') \cap (A\cap B) = \emptyset$ must hold, since $r\neq \emptyset$. Case $X'=Y$ is not possible, since due to $Y\models B(r)$, $Y\cup (A\cap B) \models B(r')$, while $Y\cup (A\cap B)\nmodels H(r')$ which contradicts $\Delta_{s_1}$. Also $\Delta_{s_2}$ prevents from having $X''\subset Y', \htm{X'',Y'} \in SE(P)$ and $X''_{\overline{A\cap B}}=X'=Y=Y'_{|\overline{A\cap B}}$. 
As for the case $X' \subset Y$,
we get for all $A' {\subseteq}\, A\cap B$, $X\cup A' \nmodels H(r'^{Y\cup A})$, which contradicts  $\htm{X \cup (A\cap B),Y\cup A'} \in SE^B(P)$.

Case (2): This means there is some $\htm{Y',Y} \in \SE^{\overline{A\cap B}}(P_{|\overline{A\cap B}})$ with $Y'\subset Y$ and $Y'\cap (B\setminus A)=Y\cap (B\setminus A)$. We claim that some $\htm{Y'\cup A'',Y\cup A'} \in SE(P)$ for $A''\supseteq A\cap B$ and $Y'\cup A'' \subset Y\cup A'$, should hold (which would then contradict $\htm{X \cup (A\cap B),Y\cup A'} \in SE^B(P)$). Assume there is a rule $r \in P$ that prevents this. So $Y'\cup A'' \nmodels r^{Y\cup A'}$ for any $A'' \subseteq A', A'' \supseteq A\cap B$. We have $Y'\cup A'' \models B(r^{Y\cup A'})$ and $Y'\cup A'' \nmodels H(r^{Y\cup A'})$. Observe that $H(r) \cap (A\cap B) \neq \emptyset$ with $Y'\cup A'' \nmodels H(r^{Y\cup A'})$ is not possible since $A\cap B \subseteq Y'$. Also $B^-(r) \cap (A\cap B) \neq \emptyset$ is not possible, since $r \in P^{Y\cup A'}$ and $A\cap B \in Y\cup A'$. Thus $r_{|\overline{A\cap B}}$ is non-empty and appears in $P_{|\overline{A\cap B}}$, which would then prevent $\htm{Y',Y} \in \SE^{\overline{A\cap B}}(P_{|\overline{A\cap B}})$, thus reaches a contradiction.

$\SE^B(P)_{|\overline{A\cap B}} = \SE^B_{A\cap B}(P)$ follows from the observation that there cannot exist different total $B$-SE-models which agree on the projection onto $\overline{A \cap B}$ as $A\cap B \subseteq B$. Thus $\SE^B_{A\cap B}(P)$ simply collects the projection of the $B$-SE-models.
\end{proof}
\end{document}